\documentclass[a4paper,11pt]{article}
\usepackage[top=4cm, bottom=4cm, left=3cm, right=3cm]{geometry}

\usepackage[utf8]{inputenc} % allow utf-8 input
\usepackage[T1]{fontenc}    % use 8-bit T1 fonts
\usepackage{url}            % simple URL typesetting
\usepackage{amsfonts}       % blackboard math symbols
\usepackage{nicefrac}       % compact symbols for 1/2, etc.
\usepackage{microtype}      % microtypography

\usepackage{amssymb}
\usepackage{amsmath}
\usepackage{amsthm}

\usepackage{thmtools, thm-restate}
\usepackage{mathrsfs} %pour mathscr
\usepackage[usenames,dvipsnames]{pstricks}
\usepackage{euler}
\usepackage{euscript}
\usepackage{graphicx}
\usepackage{charter}
\usepackage{mdframed}
\usepackage{enumerate, subfigure, color}
\usepackage[colorlinks,hyperindex]{hyperref}
\usepackage{bibunits}

\usepackage{newfloat}
\usepackage{caption}

\usepackage{booktabs}
\usepackage{cite}

\usepackage{floatrow}

\newcommand{\X}{\mathcal{X}}
\newcommand{\Y}{\mathcal{Y}}

\newcommand{\E}{\mathcal{E}}

\newcommand{\B}{\mathcal{B}}
\newcommand{\HS}{\B_2}

\renewcommand{\H}{\mathcal{H}}

\newcommand{\HY}{{\H_\Y}}
\newcommand{\HX}{{\H_\X}}
\newcommand{\HYX}{{\HY \otimes \H_\X}}

\newcommand{\R}{\mathbb{R}}
\newcommand{\G}{\Gamma}

\newcommand{\N}{\mathbb{N}}
\newcommand{\Z}{\mathbb{Z}}
\newcommand{\CC}{\mathbb{C}}

\newcommand{\mG}{\mathcal{G}}

\newcommand{\K}{\mathbf{K}}
\newcommand{\Kx}{{\K_x}}

\newcommand{\la}{\lambda}

\newcommand{\ot}{\otimes}

\newcommand{\loss}{\bigtriangleup}

\newcommand{\decoding}{d}

\newcommand{\refref}[2]{#1.~\ref{#2}}

\newcommand{\secref}[1]{\refref{Sec}{#1}}

\newcommand{\exref}[1]{Example~\ref{#1}}

\newcommand{\figref}[1]{\refref{Fig}{#1}}
\newcommand{\tabref}[1]{\refref{Tab}{#1}}

\renewcommand{\eqref}[1]{Eq.~(\ref{#1})}
\newcommand{\algref}[1]{Alg.~\ref{#1}}

\newcommand{\rhox}{{\rho_\X}}
\newcommand{\LX}{L^2(\X,\rhox)}
\newcommand{\LXR}{L^2(\X,\rhox,\R)}
\newcommand{\LXH}{L^2(\X,\rhox,\HY)}

\newcommand{\domrho}{D_{\rho|\X}}

\newcommand{\ip}[2]{\langle#1,#2\rangle}
\newcommand{\ipv}{\ip{\cdot}{\cdot}}

\newcommand{\minimize}[1]{\underset{#1}{\rm minimize}}

\newcommand{\Q}{Q}

\newcommand{\eqals}[1]{\begin{align*}#1\end{align*}}

\newcommand{\eqal}[1]{\begin{align}#1\end{align}}

\renewcommand{\eqals}[1]{\eqal{#1}}
\newcommand{\outkern}{h}

\newcommand{\tr}{\ensuremath{\text{\rm Tr}}}
\newcommand{\Span}{\ensuremath{\text{\rm span}}}
\newcommand{\ran}{\ensuremath{\text{\rm Ran}}}

\newcommand{\argmin}{\operatornamewithlimits{argmin}}
\newcommand{\argmax}{\operatornamewithlimits{argmax}}

\declaretheorem[name=Theorem,refname=Thm.]{theorem}
\declaretheorem[name=Lemma,sibling=theorem]{lemma}
\declaretheorem[name=Proposition,refname=Prop.,sibling=theorem]{proposition}
\declaretheorem[name=Remark]{remark}
\declaretheorem[name=Corollary,refname=Cor.,sibling=theorem]{corollary}
\declaretheorem[name=Definition,refname=Def.,sibling=theorem]{definition}

\declaretheorem[name=Example]{example}

\title{\sffamily\huge\bf A Consistent Regularization Approach \\ for Structured Prediction}
\author{Carlo Ciliberto $^{*,1}$ \\ {\small\em cciliber@mit.edu~~~~~} \and Alessandro Rudi $^{*,1,2}$ \\ {\small\em ale\_rudi@mit.edu~~~~~~~} \\ $ $ \and Lorenzo Rosasco $^{1,2}$ \\ {\small\em lrosasco@mit.edu~~~~~} \and {\small ${}^*$Equal contribution} }

\begin{document}

\maketitle

\begin{abstract}
\noindent We propose and analyze a regularization approach for structured prediction problems. We characterize a large class of loss functions that allows to naturally embed structured outputs in a linear space. 
We exploit this fact to  design learning  algorithms using a surrogate loss approach and regularization techniques.  
We prove universal consistency and finite sample bounds characterizing the generalization properties of the proposed methods. Experimental results are provided to demonstrate the practical usefulness of the proposed approach.
\end{abstract}

\begin{bibunit}[unsrt]
\section{Introduction}

\footnotetext[1]{Laboratory for Computational and Statistical Learning - Istituto Italiano di Tecnologia, Genova, Italy \& Massachusetts Institute of Technology, Cambridge, MA 02139, USA.}\footnotetext[2]{Universit\`a degli Studi di Genova, Genova, Italy.}
Many machine learning  applications   require dealing with  data-sets having  complex structures, e.g.   natural language processing, image segmentation, reconstruction or captioning, pose estimation, protein folding prediction to name a few.
 \cite{felzenszwalb2010object,karpathy2015deep,bakir2007}. Structured prediction problems pose a challenge 
for classic off-the-shelf learning algorithms for regression or binary classification. Indeed, this has motivated the extension of methods such as support vector machines to structured problems \cite{tsochantaridis2005}.  
% RNN?
Dealing with structured prediction problems is also a challenge for learning theory. While the theory of empirical risk minimization provides a very general statistical framework, in practice computational considerations make things more involved. Indeed, in the last few years, an effort has been made to analyze specific structured problems, such as multiclass classification \cite{mroueh2012}, multi-labeling \cite{gao2013}, ranking~\cite{duchi2010} or quantile estimation \cite{steinwart2011}. A natural question is then whether a unifying learning theoretic framework can be developed encompassing a wide range of problems as special cases.

In this paper we take a step in this direction  proposing  and analyzing  a
regularization  approach to   a wide class of structured prediction 
problems defined by  loss  functions  satisfying  mild conditions. Indeed, our starting observation is that  a large class of loss functions  naturally define an embedding of  the structured outputs in a linear space. This fact allows  to define a (least squares) surrogate loss and  cast the problem in a multi-output regularized learning framework \cite{micchelli2004,caponnetto2007,alvarez2012}. The corresponding algorithm essentially reduces to a form of kernel ridge regression and generalizes the approach proposed in \cite{cortes2005}, hence providing also a novel derivation for this latter algorithm. Our theoretical analysis allows to characterize the generalization properties of the proposed approach. In particular, it allows to quantify the impact due to the surrogate approach and establishes universal consistency as well as finite sample bounds. An experimental analysis shows promising results on a variety of structured prediction problems.

The rest of this paper is organized as follows: in \secref{sec:statement&algorithm} we introduce the structured prediction problem in its generality and present our algorithm to approach it. In \secref{sec:surrogate} we introduce and discuss a surrogate framework for structured prediction, from which we derive our algorithm. In \secref{sec:analysis}, we analyze the theoretical properties of the proposed algorithm. In \secref{sec:connections} we draw connections with previous work in structured prediction while in \secref{sec:experiments} we report a preliminary empirical analysis and comparison of the proposed approach. Finally, \secref{sec:conclusions} concludes the paper outlining relevant directions for future research.

\section{A Regularization Approach to Structured prediction}\label{sec:statement&algorithm}

%In this work we are interested in learning functional relations $f:\X\to\Y$ between sets $\X,\Y$ of {\it structured objects}, given a finite number of training examples. In particular, we refer to this problem as {\it structured prediction} when 

The goal of supervised learning is to learn functional relations $f:\X\to\Y$ between two sets $\X,\Y$, given a finite number of examples. In particular in this work we are interested to {\it structured prediction}, namely the case where $\Y$ is a set of structured outputs (such as histograms, graphs, time sequences, points on a manifold, etc.). Moreover, structure on $\Y$ can be implicitly induced by a suitable loss $\loss:\Y\times\Y\to\R$ (such as edit distance, ranking error, geodesic distance, indicator function of a subset, etc.). Then, the problem of structured prediction becomes
\begin{equation}\label{eq:expected_risk_minimization}
   \minimize{f:\X\to\Y} \quad \E(f), \qquad \mbox{with} \qquad \E(f) = \int_{\X\times\Y} \loss(f(x),y) ~~ d\rho(x,y)
\end{equation}
and the goal is to find a good estimator for the minimizer of the above equation, given a finite number of (training) points $\{(x_i,y_i)\}_{i=1}^n$ sampled from a unknown probability distribution $\rho$ on $\X\times\Y$. In the following we introduce an estimator $\hat{f}:\X\to\Y$ to approach \eqref{eq:expected_risk_minimization}. The rest of this paper is devoted to prove that $\hat{f}$ it a consistent estimator for a minimizer of \eqref{eq:expected_risk_minimization}.
\paragraph{Our Algorithm for Structured Prediction.}
In this paper we propose and analyze the following estimator
\eqal{\label{eq:algorithm}\tag{Alg.~$1$}
  \hat{f}(x) = \argmin_{y\in\Y} ~ \sum_{i=1}^n \alpha_i(x) \loss(y,y_i) \quad \mbox{with} \quad \alpha(x) = (\K + n\la I)^{-1} \K_x \in \R^n
}
given a positive definite kernel $k:\X\times \X\to \R$ and training set $\{(x_i,y_i)\}_{i=1}^n$. In the above expression, $\alpha_i(x)$ is $i$-th entry in $\alpha(x)$, $\K\in\R^{n \times n}$ is the kernel matrix $\K_{i,j} = k(x_i,x_j)$, $\K_x\in\R^n$ the vector with entires $(\K_x)_i = k(x,x_i)$, $\la>0$ a regularization parameter and $I$ the identity matrix. From a computational perspective, the procedure in~\ref{eq:algorithm} is divided in two steps: a {\it learning} step where input-dependents weights $\alpha_i(\cdot)$ are computed (which essentially consists in solving a kernel ridge regression problem) and a {\it prediction} step where the $\alpha_i(x)$-weighted linear combination in \ref{eq:algorithm} is optimized, leading to a prediction $\hat{f}(x)$ given an input $x$. This idea was originally proposed in \cite{collins2002discriminative}, where a ``score'' function $F(x,y)$ was learned to estimate the ``likelihood'' of a pair $(x,y)$ sampled from $\rho$, and then used in
\eqal{
  \hat{f}(x) = \argmin_{y\in\Y} ~ - F(x,y),
}
to predict the best $\hat{f}(x)\in\Y$ given $x\in\X$. This strategy was extended in \cite{tsochantaridis2005} for the popular {\it SVMstruct} and adopted also in a variety of approaches for structured prediction \cite{felzenszwalb2010object,cortes2005,kadri2013}.\\

\paragraph{Intuition.} While providing a principled derivation of \ref{eq:algorithm} for a large class of loss functions is a main contribution of this work, it is useful to first consider the special case where $\loss$ is induced by a reproducing kernel $\outkern:\Y\times\Y\to\R$ on the output set, such that
\eqal{\label{eq:ker_asm}
\loss(y,y') = \outkern(y,y) - 2 \outkern(y,y') + \outkern(y',y').
}
This choice of $\loss$ was originally considered in {\it Kernel Dependency Estimation (KDE)}~\cite{weston2002}. In particular, for the case of normalized kernels (i.e. $\outkern(y,y) = 1  ~~ \forall y\in\Y$), \ref{eq:algorithm} essentially reduces to \cite{cortes2005,kadri2013} and recalling their derivation is insightful. Note that, since a kernel can be written as $\outkern(y,y') = \ip{\psi(y)}{\psi(y')}_\HY$, with $\psi:\Y\to\HY$ a non-linear map into a feature space $\HY$~\cite{berlinet2011}, then \eqref{eq:ker_asm} can be rewritten as
\eqal{\label{eq:ker_asm2}
  \loss(f(x),y') = \|\psi(f(x)) - \psi(y')\|_\HY^2.
}
Directly minimizing the equation above with respect to $f$ is generally challenging due to the non linearity $\psi$. A possibility is to replace $\psi \circ f$ by a function $g:\X\to\HY$ that is easier to optimize. We can then consider the regularized problem
\eqal{\label{eq:kde}
  \minimize{g\in\mG} ~ \frac{1}{n} \sum_{i=1}^n \|g(x_i) - \psi(y_i)\|_\HY^2 + \lambda \|g\|_\mG^2
}
with $\mG$ a space of functions\footnote{$\mG$ is the reproducing kernel Hilbert space for vector-valued functions~\cite{micchelli2004} with inner product $\ip{k(x_i,\cdot)c_i}{k(x_j,\cdot)c_j}_{\mG} = k(x_i,x_j)\ip{c_i}{c_j}_\HY$} $g:\X\to\HY$ of the form $g(x) = \sum_{i=1} k(x,x_i) c_i$ with $c_i\in\HY$ and $k$ a reproducing kernel. Indeed, in this case the solution to \eqref{eq:kde} is
\eqal{\label{eq:kde_solution}
  \hat{g}(x) = \sum_{i=1}^n \alpha_i(x) \psi(y_i) \quad \mbox{with} \quad \alpha(x) = (\K + n \la I)^{-1} \K_x \in \R^n
}
where the $\alpha_i$ are the same as in \ref{eq:algorithm}. Since we replaced $\loss(f(x),y)$ by $\|g(x)-\psi(y)\|_\HY^2$, a natural question is how to recover an estimator $\hat{f}$ from $\hat{g}$. In \cite{cortes2005} it was proposed to consider
\eqal{\label{eq:decoding_kde}
  \hat{f}(x) = \argmin_{y\in\Y} ~ \|\psi(y) - \hat{g}(x)\|_\HY^2 =  \argmin_{y\in\Y} ~ \outkern(y,y) - 2 \sum_{i=1}^n \alpha_i(x)\outkern(y,y_i),
}
which corresponds to \ref{eq:algorithm} when $\outkern$ is a normalized kernel.

The discussion above provides an intuition on how \ref{eq:algorithm} is derived but raises also a few questions. First, it is not clear if and how the same strategy could be generalized to loss functions that do not satisfy \eqref{eq:ker_asm}. Second, the above reasoning hinges on the idea of replacing $\hat{f}$ with $\hat{g}$ (and then recovering $\hat{f}$ by \eqref{eq:decoding_kde}), however it is not clear whether this approach can be justified theoretically. Finally, we can ask what are the statistical properties of the resulting algorithm. We address the first two questions in the next section, while the rest of the paper is devoted to establish universal consistency and generalization bounds for algorithm~\ref{eq:algorithm}.

\paragraph{Notation and Assumptions.} We introduce here some minimal technical assumptions that we will use throughout this work. We will assume $\X$ and $\Y$ to be Polish spaces, namely separable completely metrizable spaces equipped with the associated Borel sigma-algebra. Given a Borel probability distribution $\rho$ on $\X\times\Y$ we denote with $\rho(\cdot|x)$ the associated conditional measure on $\Y$ (given $x\in\X$) and with $\rho_\X$ the marginal distribution on $\X$.

\section{Surrogate Framework and Derivation}\label{sec:surrogate}
To derive \ref{eq:algorithm} we consider ideas from surrogate approaches \cite{bartlett2006,steinwart2008,duchi2010} and in particular \cite{mroueh2012}. The idea is to tackle \eqref{eq:expected_risk_minimization} by substituting $\loss(f(x),y)$ with a ``relaxation'' $L(g(x),y)$ that is easy to optimize. The corresponding surrogate problem is
\begin{equation}\label{eq:surrogate}
    \minimize{g:\X\to\HY} \quad \mathcal{R}(g), \qquad \mbox{with} \qquad \mathcal{R}(g) = \int_{\X\times\Y} L(g(x),y) ~~ d\rho(x,y),
\end{equation}
and the question is how a solution $g^*$ for the above problem can be related to a minimizer $f^*$ of \eqref{eq:expected_risk_minimization}. This is made possible by the requirement that there exists a {\it decoding} $\decoding:\HY\to\Y$, such that
  \eqal{
    & \textrm{\it Fisher Consistency:} ~~~ \E(\decoding \circ g^*)  = \E(f^*), \label{eq:fisher_init}\\ 
    & \textrm{\it Comparison Inequality:} ~~~ \E(\decoding \circ g) - \E(f^*) \leq \varphi(\mathcal{R}(g) - \mathcal{R}(g^*)), \label{eq:comparison_init}\qquad\qquad\qquad
  }
hold for all $g:\X\to\HY$, where $\varphi:\R\to\R$ is such that $\varphi(s)\to0$ for $s\to0$. Indeed, given an estimator $\hat{g}$ for $g^*$, we can ``decode'' it considering $\hat{f} = \decoding\circ \hat{g}$ and use the {\it excess risk} $\mathcal{R}(\hat{g}) - \mathcal{R}(g^*)$ to control $\E(\hat{f})- \E(f^*)$ via the comparison inequality in \eqref{eq:comparison_init}. In particular, if $\hat{g}$ is a data-dependent predictor trained on $n$ points and $\mathcal{R}(\hat{g})\to\mathcal{R}(g^*)$ when $n\to+\infty$, we automatically have $\E(\hat{f})\to\E(f^*)$. Moreover, if $\varphi$ in \eqref{eq:comparison_init} is known explicitly, generalization bounds for $\hat{g}$ are automatically extended to $\hat{f}$.
 
Provided with this perspective on surrogate approaches, here we revisit the discussion of \secref{sec:statement&algorithm} for the case of a loss function induced by a kernel $\outkern$. Indeed, by assuming the surrogate $L(g(x),y) = \|g(x) - \psi(y)\|_\HY^2$, \eqref{eq:kde} becomes the empirical version of the surrogate problem at \eqref{eq:surrogate} and leads to an estimator $\hat{g}$ of $g^*$ as in \eqref{eq:kde_solution}. Therefore, the approach in \cite{cortes2005,kadri2013} to recover $\hat{f}(x) = \argmin_y L(g(x),y)$ can be interpreted as the result $\hat{f}(x) = \decoding \circ \hat{g}(x)$ of a suitable decoding of $\hat{g}(x)$. An immediate question is whether the above framework satisfies \eqref{eq:fisher_init} and (\ref{eq:comparison_init}). Moreover, we can ask if the same idea could be applied to more general loss functions.

In this work we identify conditions on $\loss$ that are satisfied by a large family of functions and moreover allow to design a surrogate framework for which we prove \eqref{eq:fisher_init} and (\ref{eq:comparison_init}). The first step in this direction is to introduce the following assumption.

\begin{restatable}{assumption}{Amain}\label{assumption:main}
There exists a separable Hilbert space $\HY$ with inner product $\ip{\cdot}{\cdot}_\HY$, a continuous embedding $\psi:\Y\to\HY$ and a bounded linear operator $V:\HY\to\HY$, such that
\begin{equation}\label{eq:linearized_loss}
    \loss(y,y') = \ip{\psi(y)}{V\psi(y')}_\HY \qquad \forall y,y'\in\Y
\end{equation}
\end{restatable}
\autoref{assumption:main} is similar to \eqref{eq:ker_asm2} and in particular to the definition of a reproducing kernel. Note however that by not requiring $V$ to be positive semidefinite (or even symmetric), we allow for a surprisingly wide range of functions. Indeed, below we discuss some examples of functions that satisfy \autoref{assumption:main} (see Appendix \secref{sec:losses} for more details):
\begin{example}\label{example:general}\em
The following functions of the form $\loss:\Y\times\Y\to\R$ satisfy \autoref{assumption:main}:
\begin{enumerate}
\item{\em Any loss on $\Y$ of finite cardinality}. Several problems belong to this setting, such as Multi-Class Classification, Multi-labeling, Ranking, predicting Graphs (e.g. protein foldings).
\item{\em Regression and Classification Loss Functions}: Least-squares, Logistic, Hinge, $\epsilon$-insensitive, $\tau$-Pinball.
\item{\em Robust Loss Functions} Most loss functions used for {\it robust estimation}~\cite{huber2011} such as the absolute value, Huber, Cauchy, German-McLure, ``Fair'' and $L_2-L1$. See \cite{huber2011} or the Appendix for their explicit formulation.
\item{\em KDE}. Loss functions $\loss$ induced by a kernel such as in \eqref{eq:ker_asm}.
\item{\em Distances on Histograms/Probabilities}. The $\chi^2$ and the squared Hellinger distances.
\item{\em Diffusion distances on Manifolds}. The squared diffusion distance induced by the heat kernel (at time $t>0$) on a compact Reimannian manifold without boundary~\cite{schoen1994}.
\end{enumerate}
\end{example}
\paragraph{The Least Squares Loss Surrogate Framework.} \autoref{assumption:main} implicitly defines the space $\HY$ similarly to \eqref{eq:ker_asm2}. The following result motivates the choice of the least squares surrogate and moreover suggests a possible choice for the decoding.
\begin{restatable}{lemma}{Lrelation}\label{lm:relation_expected_risk}
Let $\loss:\Y\times\Y\to\R$ satisfy \autoref{assumption:main} with $\psi:\Y\to\HY$ bounded. Then the expected risk in \eqref{eq:expected_risk_minimization} can be written as
\eqal{\label{eq:relation_expected_risk}
  \E(f) = \int_\X \ip{\psi(f(x))}{V g^*(x)}_\HY ~ d\rhox(x)
}
for all $f:\X\to\Y$, where $g^*:\X\to\HY$ minimizes
\begin{equation}\label{eq:surrogate_ls}
    \mathcal{R}(g) = \int_{\X\times\Y} \|g(x) - \psi(y)\|_\HY^2 ~~ d\rho(x,y).
\end{equation}
\end{restatable}
\autoref{lm:relation_expected_risk} shows how \eqref{eq:surrogate_ls} arises naturally as surrogate problem. In particular, \eqref{eq:relation_expected_risk} suggests how to chose the decoding. Indeed, consider an $f:\X\to\Y$ such that for each $x\in\X$, $f(x)$ is a minimizer of the argument in \eqref{lm:relation_expected_risk}, namely $\ip{\psi(f(x))}{V g^*(x)}_\HY \leq \ip{\psi(y}{V g^*(x)}_\HY$ for all $y\in\Y$. Then we have $\E(f) \leq \E(f')$ for any $f':\X\to\Y$. Following this observation, in this work we consider the decoding $\decoding:\HY\to\Y$ such that
\eqal{\label{eq:decoding_function}
  \decoding(h) = \argmin_{y\in\Y} ~ \ip{~\psi(y)~}{~V h~}_\HY \qquad \forall h\in\HY.
}
So that $\E(\decoding \circ g^*) \leq \E(f')$ for all $f':\X\to\Y$. Indeed, for this choice of decoding, we have the following result.
\begin{restatable}{theorem}{Tsurrogate}\label{teo:surrogate}
Let $\loss:\Y\times\Y\to\R$ satisfy \autoref{assumption:main} with $\Y$ a compact set. Then, for every measurable $g:\X\to\HY$ and $d:\HY\to\Y$ satisfying \eqref{eq:decoding_function}, the following holds
\eqal{
    & \E(\decoding \circ g^*) = \E(f^*) \\ 
    & \E(\decoding \circ g) - \E(f^*) \leq 2 c_\loss \sqrt{\mathcal{R}(g) - \mathcal{R}(g^*)
}. \label{eq:comparison_inequality_ls}
}
with $c_\loss = \|V\| \max_{y\in\Y} \|\psi(y)\|_\HY$.
\end{restatable}
\autoref{teo:surrogate} shows that for all $\loss$ satisfying \autoref{assumption:main}, the corresponding surrogate framework identified by the surrogate in \eqref{eq:surrogate_ls} and decoding \eqref{eq:decoding_function} satisfies Fisher consistency and the comparison inequality in \eqref{eq:comparison_inequality_ls}. We recall that a finite set $\Y$ is always compact, and moreover, assuming the discrete topology on $\Y$, we have that any $\psi:\Y\to\HY$ is continuous. Therefore, \autoref{teo:surrogate} applies in particular to any structured prediction problem on $\Y$ with finite cardinality.

\autoref{teo:surrogate} suggest to approach structured prediction by first learning $\hat{g}$ and then decoding it to recover $\hat{f} = \decoding \circ \hat{g}$. A natural question is how to choose $\hat{g}$ in order to compute $\hat{f}$ in practice. In the rest of this section we propose an approach to this problem.

\paragraph{Derivation for \ref{eq:algorithm}.} Minimizing $\mathcal{R}$ in \eqref{eq:surrogate_ls} corresponds to a vector-valued regression problem \cite{micchelli2004,caponnetto2007,alvarez2012}. In this work we adopt an empirical risk minimization approach to learn $\hat{g}$ as in \eqref{eq:kde}. The following result shows that combining $\hat{g}$ with the decoding in \eqref{eq:decoding_function} leads to the $\hat{f}$ in \ref{eq:algorithm}.
\begin{restatable}{lemma}{Pcomputable}\label{prop:computable}
Let $\loss:\Y\times\Y\to\R$ satisfy \autoref{assumption:main} with $\Y$ a compact set. Let $\hat{g}:\X\to\HY$ be the minimizer of \eqref{eq:kde}. Then, for all $x\in\X$
\begin{equation}\label{eq:algorithm_with_weights}
    \decoding \circ \hat{g} (x) ~=~ \argmin_{y\in\Y} \sum_{i=1}^n \alpha_i(x) \loss(y,y_i) \qquad \alpha(x) = (\K + n \lambda I)^{-1} \K_x \in\R^n
\end{equation}
\end{restatable}
\autoref{prop:computable} concludes the derivation of \ref{eq:algorithm}. An interesting observation is that computing $\hat{f}$ does not require explicit knowledge of the embedding $\psi$ and the operator $V$, which are implicitly encoded within the loss $\loss$ by \autoref{assumption:main}. In analogy to the {\em kernel trick}~\cite{scholkopf2002} we informally refer to such assumption as the ``loss trick''. We illustrate this effect with an example.
\begin{example}[Ranking]\label{example:ranking}
In ranking problems the goal is to predict ordered sequences of a fixed number $\ell$ of labels. For these problems, $\Y$ corresponds to the set of all ordered sequences of $\ell$ labels and has cardinality $|\Y|= \ell!$, which is typically dramatically larger than the number $n$ of training examples (e.g. for $\ell=15$, $\ell! \simeq 10^{12}$). Therefore, given an input $x\in\X$, directly computing $\hat{g}(x)\in\R^{|\Y|}$ is impractical. On the opposite, the loss trick allows to express $\decoding \circ \hat{g}(x)$ only in terms of the $n$ weights $\alpha_i(x)$ in \ref{eq:algorithm}, making the computation of the {$\mbox{argmin}$} easier to approach in general. For details on the rank loss $\loss_{rank}$ and the corresponding optimization over $\Y$ we refer to the empirical analysis of \secref{sec:experiments}.
\end{example}
% This result is of practical relevance since the dimension of $\HY$ can be extremely large (or even infinite) and it can be impractical to compute $\hat{g}$. On the contrary, \autoref{prop:computable} guarantees that the decoding $\hat{f} = \decoding \circ \hat{g}$ is instead computationally tractable, scaling linearly with the number of training points. 
%
% \begin{remark}[The ``Loss Trick'']
% Note that by \autoref{prop:computable}, computing $\hat{f}$ does not require explicit knowledge of the surrogate framework. Indeed, the space $\HY$, the embedding $\psi$ and the operator $V$ are implicitly encoded in the loss by \autoref{assumption:main}, to which we refer informally as the ``loss trick'' in analogy to the {\em kernel trick}~\cite{scholkopf2002}.
% \end{remark}
%
In this section we have shown a derivation for the structured prediction algorithm proposed in this work. In \autoref{teo:surrogate} we have shown how the expected risk of the proposed estimator $\hat{f}$ is related to an estimator $\hat{g}$ via a comparison inequality. In the following we will make use of these results to prove consistency and generalization bounds for \algref{eq:algorithm}.

% Now that we have shown how to compute $\hat{f}$ from $\hat{g}$, a natural question is if we can characterize the generalization properties of $\hat{g}$ 
%
%  can be computed in practice, a natural question is how combining the comparison inequality of \autoref{teo:surrogate} with the generalization properties of $\hat{g}$ are translated to $\hat{f}$ via the comparison inequality of \autoref{teo:surrogate}. We address this question in the following section where we prove universal consistency and generalization bounds for $\hat{f}$.
%
%
\section{Statistical Analysis}\label{sec:analysis}
In this section we study the statistical properties of \ref{eq:algorithm}. In particular we made use of the relation between the structured and surrogate problems via the comparison inequality in \autoref{teo:surrogate}. We begin our analysis by proving that \ref{eq:algorithm} is {\it universally consistent}.
\begin{restatable}[Universal Consistency]{theorem}{Tuniversal}\label{teo:universal_consistency}
Let $\loss:\Y\times\Y\to\R$ satisfy \autoref{assumption:main}, $\X$ and $\Y$ be compact sets and $k:\X\times\X\to\R$ a continuous universal reproducing kernel\footnote{This is a standard assumption for universal consistency (see~\cite{steinwart2008}). An example of continuous universal kernel is the Gaussian $k(x,x') = \exp(-\|x-x'\|^2/\sigma)$.}. For any $n\in\N$ and any distribution $\rho$ on $\X\times\Y$ let $\hat{f}_n:\X\to\Y$ be obtained by \ref{eq:algorithm} with $\{(x_i,y_i)\}_{i=1}^n$ training points independently sampled from $\rho$ and $\lambda_n = n^{-1/4}$. Then,
\eqal{\label{eq:universal_consistency}
\lim_{n \to +\infty} \E(\hat f_n) = \E(f^*)  \qquad \mbox{with probability} ~ 1
}
\end{restatable}
\autoref{teo:universal_consistency} shows that, when the $\loss$ satisfies \autoref{assumption:main}, \ref{eq:algorithm} approximates a solution $f^*$ to \eqref{eq:expected_risk_minimization} arbitrarily well, given a sufficient number of training examples. To the best of our knowledge this is the first consistency result for structured prediction in the general setting considered in this work and characterized by \autoref{assumption:main}, in particular for the case of $\Y$ with infinite cardinality (dense or discrete). 

The {\it No Free Lunch} Theorem~\cite{wolpert1996lack} states that it is not possible to prove uniform convergence rates for \eqref{eq:universal_consistency}. However, by imposing suitable assumptions on the regularity of $g^*$ it is possible to prove generalization bounds for $\hat{g}$ and then, using \autoref{teo:surrogate}, extend them to $\hat{f}$. To show this, it is sufficient to require that $g^*$ belongs to $\mathcal{G}$ the reproducing kernel Hilbert space used in the ridge regression of \eqref{eq:kde}. Note that in the proofs of \autoref{teo:universal_consistency} and \autoref{teo:simple_bound}, our analysis on $\hat{g}$ borrows ideas from \cite{caponnetto2007} and extends their result to our setting for the case of $\HY$ infinite dimensional (i.e. when $\Y$ has infinite cardinality). Indeed, note that in this case \cite{caponnetto2007} cannot be applied to the estimator $\hat{g}$ considered in this work (see Appendix \autoref{sec:appendix_bounds}, \autoref{lemma:prob-bound} for details).
\begin{restatable}[Generalization Bound]{theorem}{Tsimplebound}\label{teo:simple_bound}
Let $\loss:\Y\times\Y\to\R$ satisfy \autoref{assumption:main}, $\Y$ be a compact set and $k:\X\times\X\to\R$ a bounded continuous reproducing kernel. Let $\hat{f}_n$ denote the solution of \ref{eq:algorithm} with $n$ training points and $\la = n^{-1/2}$. If the surrogate risk $\mathcal{R}$ defined in \eqref{eq:surrogate_ls} admits a minimizer $g^*\in\mG$, then
\eqal{
\E(\hat f_n) - \E(f^*)  \leq c \tau^2 ~ n^{-\frac{1}{4}}
}
holds with probability $1 - 8e^{-\tau}$ for any $\tau > 0$, with $c$ a constant not depending on $n$ and $\tau$.
\end{restatable}
The bound in \eqref{teo:simple_bound} is of the same order of the generalization bounds available for the least squares binary classifier\cite{yao2007early}. Indeed, in \secref{sec:connections} we show that in classification settings \ref{eq:algorithm} reduces to least squares classification.
\begin{remark}[Better Comparison Inequality]\label{rem:comparison}
The generalization bounds for the least squares classifier can be improved by imposing regularity conditions on $\rho$ via the {\it Tsybakov condition} \cite{yao2007early}. This was observed in \cite{yao2007early} for binary classification with the least squares surrogate, where a tighter comparison inequality than the one in \autoref{teo:surrogate} was proved. Therefore, a natural question is whether the inequality of \autoref{teo:surrogate} could be similarly improved, consequently leading to better rates for \autoref{teo:simple_bound}. Promising results in this direction can be found in \cite{mroueh2012}, where the Tsybakov condition was generalized to the multi-class setting and led to a tight comparison inequality analogous to the one for the binary setting. However, this question deserves further investigation. Indeed, it is not clear how the approach in \cite{mroueh2012} could be further generalized to the case where $\Y$ has infinite cardinality.
\end{remark}
\begin{remark}[Other Surrogate Frameworks]
In this paper we focused on a least squares surrogate loss function and corresponding framework. A natural question is to ask whether other loss functions could be considered to approach the structured prediction problem, sharing the same or possibly even better properties. This question is related also to \autoref{rem:comparison}, since different surrogate frameworks could lead to sharper comparison inequalities. This seems an interesting direction for future work.
\end{remark}
%
%Note that for $\HY$ of finite dimension, the surrogate problem reduces to multi-task learning \cite{micchelli2004}. In this case, the results on consistency and generalization bounds in \autoref{teo:universal_consistency} and \ref{teo:simple_bound} follow directly by applying the comparison inequality in \autoref{teo:surrogate} to the results in \cite{micchelli2006universal,caponnetto2007}. However, in the most typical case of $\HY$ infinite dimensional (e.g. when $\Y$ is a dense set), the results in \cite{caponnetto2007} do not directly apply and therefore we extended them to our setting (see supplementary material, \autoref{sec:appendix_bounds}, \autoref{lemma:prob-bound}).
%
\section{Connection with Previous Work}\label{sec:connections}
In this section we draw connections between \ref{eq:algorithm} and previous methods for structured prediction learning. 
\paragraph{Binary and Multi-class Classification.}
It is interesting to note that in classification settings, \ref{eq:algorithm} corresponds to the least squares classifier~\cite{yao2007early}. Indeed, let $\Y = \{1,\dots,\ell\}$ be a set of labels and consider the {\it misclassification} loss $\loss(y,y') = 1$ for $y\neq y'$ and $0$ otherwise. Then $\loss(y,y') = e_y^\top V e_{y'}$ with $e_i\in\R^\ell$ the $i$-the element of the canonical basis of $\R^\ell$ and $V = \mathbf{1} - I$, where $I$ is the $\ell \times \ell$ identity matrix and $\mathbf{1}$ the matrix with all entries equal to $1$. In the notation of surrogate methods adopted in this work, $\HY = \R^\ell$ and $\psi(y) = e_y$. Note that both Least squares classification and our approach solve the surrogate problem at \eqref{eq:kde}
\eqal{
  \frac{1}{n} \sum_{i=1}^n \|g(x_i) - e_{y_i}\|_{\R^T}^2 + \lambda ~ \|g\|_\mG^2
}
to obtain a vector-valued predictor $\hat{g}:\X\to\R^T$ as in \eqref{eq:kde_solution}. Then, the least squares classifier $\hat{c}$ and the decoding $\hat{f} = \decoding \circ \hat{g}$ are respectively obtained by
\eqals{
  \hat{c}(x) = \argmax_{i=1,\dots,T} \hat{g}(x) \qquad\qquad\qquad \hat{f}(x) = \argmin_{i=1,\dots,T} V \hat{g}(x).
} 
However, since $V = \mathbf{1} - I$, it is easy to see that $\hat{c}(x) = \hat{f}(x)$ for all $x\in\X$.
\paragraph{Kernel Dependency Estimation.}
In \autoref{sec:statement&algorithm} we discussed the relation between KDE~\cite{weston2002,cortes2005} and \ref{eq:algorithm}. In particular, we have observed that if $\loss$ is induced by a kernel $h:\Y\times\Y\to\R$ as in \eqref{eq:ker_asm} and $h$ is normalized, i.e. $h(y,y) = 1 ~ \forall y\in\Y$, then algorithm \eqref{eq:decoding_kde} proposed in \cite{cortes2005} leads to the same predictor as \ref{eq:algorithm}. Therefore, we can apply \autoref{teo:universal_consistency} and \ref{teo:simple_bound} to prove universal consistency and generalization bounds for methods such as \cite{cortes2005,kadri2013}. We are not aware of previous results proving consistency (and generalization bounds) for the KDE methods in \cite{cortes2005,kadri2013}. Note however that  when the kernel $h$ is not normalized, the ``decoding'' in \eqref{eq:decoding_kde} is not equivalent to \ref{eq:algorithm}. In particular, given the surrogate solution $g^*$, applying \eqref{eq:decoding_kde} leads to predictors that are do not minimize \eqref{eq:expected_risk_minimization}. As a consequence the approach in \cite{cortes2005} is not consistent in the general case.
\paragraph{Support Vector Machines for Structured Output.}
A popular approach to structured prediction is the {\it Support Vector Machine for Structured Outputs (SVMstruct)}~\cite{tsochantaridis2005} that extends ideas from the well-known SVM algorithm to the structured setting. One of the main advantages of SVMstruct is that it can be applied to a variety of problems since it does not impose strong assumptions on the loss. In this view, our approach, as well as KDE, shares similar properties, and in particular allows to consider $\Y$ of infinite cardinality. Moreover, we note that generalization studies for SVMstruct are available \cite{bakir2007} (Ch. $11$). However, it seems that these latter results do not allow to derive universal consistency of the method.
\section{Experiments}\label{sec:experiments}
\begin{figure}[t]
\begin{minipage}{\textwidth}
\begin{minipage}{0.4\textwidth}
\begin{table}[H]
\footnotesize
\begin{center}
\begin{tabular}{rc}
\toprule
                                                      & {\bf Rank Loss} \\
\midrule
{\bf Linear}~\cite{duchi2010}                         & $0.430 \pm 0.004$  \\
{\bf Hinge}~\cite{herbrich1999}                       & $0.432 \pm 0.008$  \\
{\bf Logistic}~\cite{dekel2004}                       & $0.432 \pm 0.012$  \\
{\bf SVM Struct}~\cite{tsochantaridis2005}            & $0.451 \pm 0.008$ \\
{\bf \ref{eq:algorithm}}              & $\mathbf{0.396 \pm 0.003}$ \\
\bottomrule
\hline
\end{tabular}
\end{center}
\caption{Comparison of ranking methods on the MovieLens dataset~\cite{harper2015} with the normalized rank loss $\loss_{rank}$.}
\label{tab:ranking}
\end{table}
\end{minipage} \quad~ ~
\begin{minipage}{0.55\textwidth}
\begin{table}[H]
\footnotesize
\begin{center}
\begin{tabular}{rcc}
\toprule
{\bf Loss}      & \shortstack{\bf KDE~\cite{weston2002} \\ (Gaussian)}   & \shortstack{\bf \ref{eq:algorithm} \\ (Hellinger)}\\
\midrule
{\bf $\loss_G$}      & $\mathbf{0.149 \pm 0.013}$         & $0.172 \pm 0.011$ \\
{\bf $\loss_H$}      & $0.736 \pm 0.032$                  & $\mathbf{0.647 \pm 0.017}$ \\
{\bf $\loss_R$}      & $0.294 \pm 0.012$                  & $\mathbf{0.193 \pm 0.015}$ \\
\bottomrule
\hline
\end{tabular}
\end{center}
\caption{Digit reconstruction on USPS with KDE method~\cite{weston2002} (with Gaussian loss) and \ref{eq:algorithm} with squared Hellinger distances. Performance measured with Gaussian, Hellinger and Reconstruction loss (see text).
\label{tab:hellinger}}
\end{table}
\end{minipage}
\end{minipage}
\end{figure}
\paragraph{Ranking Movies.} 
We considered the problem of ranking movies in the MovieLens dataset \cite{harper2015} (ratings (from $1$ to $5$) of $1682$ movies by $943$ users). The goal was to predict preferences of a given user, i.e. an ordering of the $1682$ movies, according to the user's partial ratings. Note that, as observed in \exref{example:ranking}, in ranking problems the output set $\Y$ is the collection of all ordered sequences of a predefined length. Therefore, $\Y$ is finite (albeit extremely large) and we can apply \ref{eq:algorithm}. 

We applied \ref{eq:algorithm} to the ranking problem using the {\it rank loss} \cite{duchi2010} 
\eqal{
  \loss_{rank}(y,y') = \sum_{i,j=1}^M \gamma(y')_{ij} \ (1 - {\rm sign}(y_i - y_j))/2,
}
with $M$ is the number of movies in the database, $y\in\Y$ a vector of the $M$ integers $y_i\in\{1,\dots,M\}$ without repetition, where $y_i$ corresponding to the rank assigned by $y$ to movie $i$. In the definition of $\loss_{rank}$, $\gamma(y)_{ij}$ denotes the costs (or reward) of having movie $j$ ranked higher than movie $i$ and, similarly to \cite{duchi2010}, we set $\gamma(y)_{ij}$ equal to the difference of ratings provided by user associated to $y$ (from $1$ to $5$). We chose as $k$ in \ref{eq:algorithm}, a linear kernel on features similar to those proposed in \cite{duchi2010}, which were computed based on users' profession, age, similarity of previous ratings, etc. Since solving \ref{eq:algorithm} for $\loss_{rank}$ is NP-hard (see \cite{duchi2010}) we adopted the {\it Feedback Arc Set approximation (FAS)} proposed in \cite{eades1993} to approximate the $\hat{f}(x)$ of \ref{eq:algorithm}. Results are reported in \tabref{tab:ranking} comparing \ref{eq:algorithm} (Ours) with surrogate ranking methods using a Linear~\cite{duchi2010}, Hinge~\cite{herbrich1999} or Logistic~\cite{dekel2004} loss and  Struct SVM~\cite{tsochantaridis2005}\footnote{implementation from {\small\url{http://svmlight.joachims.org/svm_struct.html}}}. We randomly sampled $n = 643$ users for training and tested on the remaining $300$. We performed $5$-fold cross-validation for model selection. We report the normalized $\loss_{rank}$, averaged over $10$ trials to account for statistical variability. Interestingly, our approach appears to outperform all competitors, suggesting that \ref{eq:algorithm} is a viable approach to ranking.
\paragraph{Image Reconstruction with Hellinger Distance.} We considered the USPS\footnote{\small\url{http://www.cs.nyu.edu/~roweis/data.html}} digits reconstruction experiment originally proposed in \cite{weston2002}. The goal is to predict the lower half of an image depicting a digit, given the upper half of the same image in input. The standard approach is to use a Gaussian kernel $k_G$ on images in input and adopt KDE methods such as \cite{weston2002,cortes2005,kadri2013} with loss $\loss_{G}(y,y') = 1 - k_G(y,y')$. Here we take a different approach and, following \cite{cuturi2013sinkhorn}, we interpret an image depicting a digit as an histogram and normalize it to sum up to $1$. Therefore, $\Y$ becomes is the unit simplex in $\R^{128}$ ($16 \times 16$ images) and we adopt the squared Hellinger distance $\loss_H$ 
\eqals{
  \loss_H(y,y') = \sum_{i=1}^n (\sqrt{y_i} - \sqrt{y_i}')^2 \qquad \mbox{for ~~~~} y = (y_i)_{i=1}^M
}
to measure distances on $\Y$. We used the kernel $k_G$ on the input space and compared \ref{eq:algorithm} using respectively $\loss_H$ and $\loss_G$. For $\loss_G$ \ref{eq:algorithm} correpsponds to \cite{cortes2005}. We performed digit reconstruction experiments by training on $1000$ examples evenly distributed among the $10$ digits of USPS and tested on $5000$ images. We performed $5$-fold cross-validation for model selection. \tabref{tab:hellinger} reports the performance of \ref{eq:algorithm} and the KDE methods averaged over $10$ runs. Performance are reported according to the Gaussian loss $\loss_{G}$ Hellinger loss $\loss_{H}$. Unsurprisingly, methods trained with respect to a specific loss perform better than the competitor with respect to such loss. Therefore, as a further measure of performance we also introduced the ``Recognition'' loss $\loss_R$. This loss has to be intended as a measure of how ``well'' a predictor was able to correctly reconstruct an image for digit recognition purposes. To this end, we trained an automatic digit classifier and defined $\loss_R$ to be the misclassification error of such classifier when tested on images reconstructed by the two prediction algorithms. This automatic classifier was trained using a standard SVM \cite{scholkopf2002} (with {\sc libSVM}\footnote{\small\url{https://www.csie.ntu.edu.tw/~cjlin/libsvm/}}) on a separate subset of USPS images and achieved an average $~0.04\%$ error rate on the true $5000$ test sets. In this case a clear difference in performance can be observed between using two different loss functions, suggesting that $\loss_H$ is more suited to the reconstruction problem.
\begin{figure}[t]
\CenterFloatBoxes
\begin{floatrow}
\begin{minipage}[t]{0.4\textwidth}
\ffigbox{% 
\hspace*{-0.35\textwidth}
\includegraphics[width=1\textwidth]{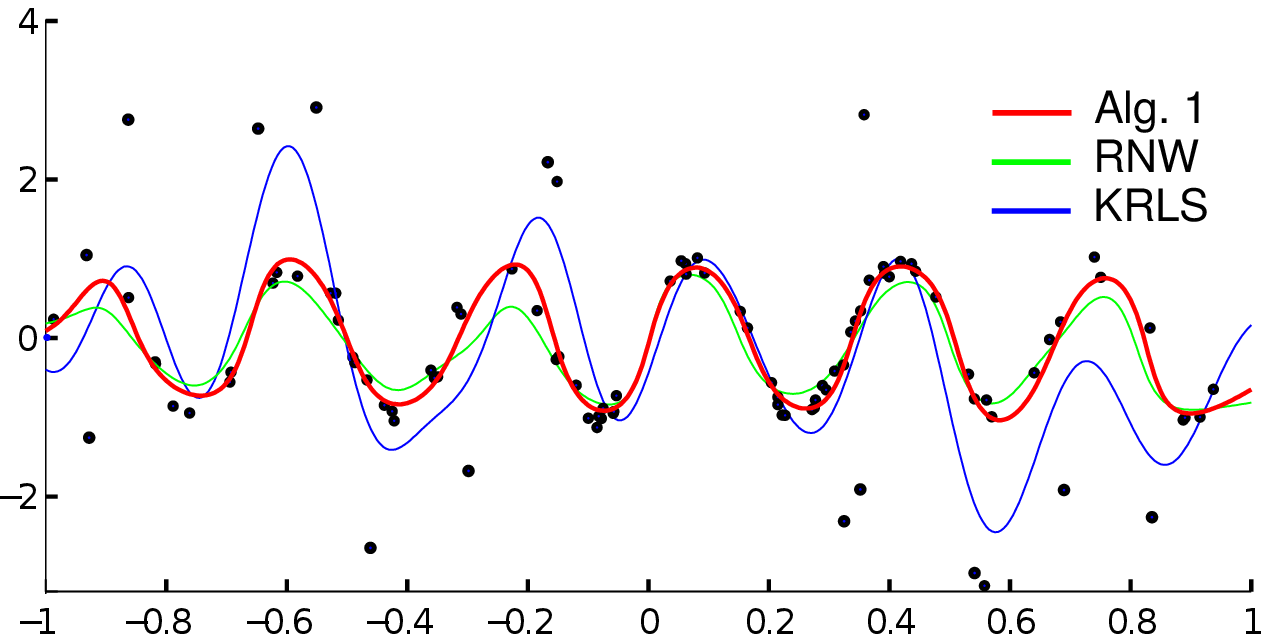}%
}{}
\end{minipage}
\begin{minipage}[t]{0.4\textwidth}
\footnotesize
\begin{tabular}{rccc}
\toprule
{\bf n} & {\bf \ref{eq:algorithm}} & {\bf RNW} & {\bf KRR} \\
\midrule
{\bf 50} & {\bf0.39} $\pm$ 0.17 & 0.45 $\pm$ 0.18 & 0.62 $\pm$ 0.13\\
{\bf 100} & {\bf0.21} $\pm$ 0.04 & 0.29 $\pm$ 0.04 & 0.47 $\pm$ 0.09\\
{\bf 200} & {\bf0.12} $\pm$ 0.02 & 0.24 $\pm$ 0.03 & 0.33 $\pm$ 0.04 \\
{\bf 500} & {\bf0.08} $\pm$ 0.01 & 0.22 $\pm$ 0.02 & 0.31 $\pm$ 0.03\\
{\bf 1000} & {\bf0.07} $\pm$ 0.01 & 0.21 $\pm$ 0.02 & 0.19 $\pm$ 0.02\\
\bottomrule
\hline
\end{tabular}
\end{minipage}
% \end{center}
\caption{Robust estimation on the regression problem in \secref{sec:experiments} by minimizing the Cauchy loss with \ref{eq:algorithm} (Ours) or Nadaraya-Watson (Nad). KRLS as a baseline predictor. {\bf Left}. Example of one run of the algorithms. {\bf Right}. Average distance of the predictors to the actual function (without noise and outliers) over $100$ runs with respect to training sets of increasing dimension.\label{fig:robust}}
\end{floatrow}
\end{figure}
\paragraph{Robust Estimation.}
We considered a regression problems with many outliers and evaluated \ref{eq:algorithm} using the Cauchy loss (see \exref{example:general} - (3)) for robust estimation. Indeed, in this setting, $\Y = [-M,M]\subset\R$ is not structured, but the non-convexity of $\loss$ can be an obstacle to the learning process. We generated a dataset according to the model $y =  \sin(6\pi x) + \epsilon + \zeta$, where $x$ was sampled uniformly on $[-1,1]$ and $\epsilon$ according to a zero-mean Gaussian with variance $0.1$. $\zeta$ modeled the outliers and was sampled according to a zero-mean random variable that was $0$ with probability $0.90$ and a value uniformly at random in $[-3,3]$ with probability $0.1$). We compared \ref{eq:algorithm} with the {\it Nadaraya-Watson} robust estimator (RNW)~\cite{hardle1984robust} and kernel ridge regression (KRR) with a Gaussian kernel as baseline. To train \ref{eq:algorithm} we used a Gaussian kernel on the input and performed predictions (i.e. solved \eqref{eq:algorithm_with_weights}) using Matlab {\sc fminunc} function for unconstrained minimization. Experiments were performed with training sets of increasing dimension ($100$ repetitions each) and test set of $1000$ examples. $5$-fold cross-validation for model selection. Results are reported in \figref{fig:robust}, showing that our estimator significantly outperforms the others. Moreover, our method appears to greatly benefit from training sets of increasing size.
\section{Conclusions and Future Work}\label{sec:conclusions}
In this work we considered the problem of structured prediction from a Statistical Learning Theory perspective. We proposed a learning algorithm for structured prediction that is split into a learning and prediction step similarly to previous methods in the literature. We studied the statistical properties of the proposed algorithm by adopting a strategy inspired to surrogate methods. In particular, we identified a large family of loss functions for which it is natural to identify a corresponding surrogate problem. This perspective allows to prove a derivation of the algorithm proposed in this work. Moreover, by exploiting a comparison inequality relating the original and surrogate problems we were able to prove universal consistency and generalization bounds under mild assumption. In particular, the bounds proved in this work recover those already known for least squares classification, of which our approach can be seen as a generalization. We supported our theoretical analysis with experiments showing promising results on a variety of structured prediction problems. 

A few questions were left opened. First, we ask whether the comparison inequality can be improved (under suitable hypotheses) to obtain faster generalization bounds for our algorithm. Second, the surrogate problem in our work consists of a vector-valued regression (in a possibly infinite dimensional Hilbert space), we solved this problem by plain kernel ridge regression but it is natural to ask whether approaches from the multi-task learning literature could lead to substantial improvements in this setting. Finally, an interesting question is whether alternative surrogate frameworks could be derived for the setting considered in this work, possibly leading to tighter comparison inequalities. We will investigate these questions in the future.
{\small
\bibliographystyle{natbib}
\putbib[biblio]
}
\end{bibunit}

\newpage

\appendix

\begin{bibunit}[unsrt]

\section*{Appendix}

The Appendix of this work is divided in the following three sections:
\begin{enumerate}[A]
\item Proofs of Fisher consistency and comparison inequality (\autoref{teo:surrogate}).
\item Universal Consistency and Generalization Bounds for \ref{eq:algorithm}. (\autoref{teo:universal_consistency} and \ref{teo:simple_bound}).
\item The characterization of a large family of $\loss$s satisfying \autoref{assumption:main} (\autoref{teo:taxonomy}).
\end{enumerate}

\section*{Mathematical Setting}

In the following we will always assume $\X$ and $\Y$ to be Polish spaces, namely separable complete metrizable spaces, equipped with the associated Borel sigma-algebra. When referring to a probability distribution $\rho$  on $\X \times \Y$ we will always assume it to be a Borel probability measure, with $\rhox$ the marginal distribution on $\X$ and $\rho(\cdot|x)$ the conditional measure on $\Y$ given $x \in \X$. We recall \cite{dudley2002real} that $\rho(y|x)$ is a regular conditional distribution and its domain, which we will denote $\domrho$ in the following, is a measurable set contained in the support of $\rhox$ and corresponds to the support of $\rhox$ up to a set of measure zero. 

For convenience, we recall here the main assumption of our work.

\Amain*

\paragraph{Basic notation}

We recall that a Hilbert space $\H$ is a vector space with inner product $\ip{\cdot}{\cdot}_\H$, closed with respect to the norm $\|h\|_\H = \sqrt{\ip{h}{h}_\H}$ for any $h\in\H$. We denote with $L^2(\X,\rhox,\H)$ the Lebesgue space of square integrable functions on $\X$ with respect to a measure $\rhox$ and with values in a separable Hilbert space $\H$. We denote with $\ip{f}{g}_{\rhox}$ the inner product $\int \ip{f(x)}{g(x)}_\H d\rhox(x)$, for all $f, g \in L^2(\X,\rhox,\H)$. In particular when $\H = \R$ we denote with $\LX$ the space $\LXR$.

Given a linear operator $V:\H\to\H'$ between two Hilbert spaces $\H,\H'$, we denote with $\tr(V)$ the trace of $V$ and with $V^*:\H'\to\H$ the adjoint operator associated to $V$, namely such that $\ip{Vh}{h'}_\H' = \ip{h}{V^*h'}_{\H'}$ for every $h\in\H$, $h'\in\H'$. Moreover, we denote with $\|V\| = \sup_{\|h\|_\H \leq 1} \|Vh\|_{\H'}$ and $\|V\|_{HS}= \sqrt{\tr(V^*V)}$ respectively the operator norm and Hilbert-Schmidt norm of $V$. We recall that a linear operator $V$ is continuous if and only if $\|V\| < +\infty$ and we denote $\B(\H,\H')$ the set of all continuous linear operators from $\H$ to $\H'$. Moreover, we denote $\HS(\H,\H')$ the set of all operators $V:\H\to\H'$ with $\|V\|_{HS} < +\infty$ and recall that $\HS(\H,\H')$ is isometric to the space $\H' \otimes \H$, with $\otimes$ denoting the tensor product. Indeed, for the sake of simplicity, with some abuse of notation we will not make the distinction between the two spaces.

Note that in most of our results we will require $\Y$ to be non-empty and compact, so that a continuous functional over $\Y$ always attains a minimizer on $\Y$ and therefore the operator $\argmin_{y\in\Y}$ is well defined. Note that for a finite set $\Y$, we will always assume it endowed with the discrete topology, so that $\Y$ is compact and any function $\loss:\Y\times\Y\to\R$ is continuous. 

\paragraph{On the Argmin} Notice that for simplicity of notation, in the paper we denoted the minimizer of \ref{eq:algorithm} as
\eqals{
  \hat{f}(x) = \argmin_{y\in\Y} \sum_{i=1}^n \alpha_i(x) \loss(y,y_i).
}
However note that the correct notation should be
\eqals{
  \hat{f}(x) \in \argmin_{y\in\Y} \sum_{i=1}^n \alpha_i(x) \loss(y,y_i)
}
since a loss function $\loss$ can have more than one minimizer in general. In the following we keep this more pedantic, yet correct notation.

\paragraph{Expected Risk Minimization} Note that whenever we write an expected risk minimization problem, we implicitly assume the optimization domain to be the space of measurable functions. For instance, \eqref{eq:expected_risk_minimization} would be written more rigorously as
\eqal{
    \mbox{minimize } \left\{\int_{\X\times\Y} \loss(f(x),y) \ d\rho(x,y) \ \left| \ f:\X\to\Y \mbox{ measurable} \right. \right\}
}

In the next Lemma, following \cite{steinwart2008} we show that the problem in \eqref{eq:expected_risk_minimization} admits a measurable pointwise minimizer.

\begin{lemma}[Existence of a solution for \eqref{eq:expected_risk_minimization}]\label{lemma:solution-structured-risk}
Let $\loss:\Y\times\Y \to \R$ be a continuous function. Then, the expected risk minimization at \eqref{eq:expected_risk_minimization} admits a measurable minimizer $f^*: \X \to \Y$ such that
\eqal{\label{eq:solution_expected_risk}
  f^*(x) \in \argmin_{y\in\Y} ~ \int_{\Y} \loss(y,y') d\rho(y'|x)
}
for every $x\in \domrho$. Moreover, the function $m: \X \to \R$ defined as follows, is measurable
\eqals{
m(x) = \inf_{y \in \Y} r(x,y), \qquad \mbox{with} \qquad r(x,y) = \left\{ \begin{array}{cc} 
                      \int_\Y \loss(y,y')d\rho(y'|x)  & \mbox{if } x \in \domrho \\
                      0                               & \mbox{otherwise}
   \end{array} \right.
}
\end{lemma}
\begin{proof}
Since $\loss$ is continuous and $\rho(y|x)$ is a regular conditional distribution, then $r$ is a Carath\'{e}odory function (see Definition $4.50$ (pp. $153$) of \cite{aliprantis2006}), namely continuous in $y$ for each $x\in\X$ and measurable in $x$ for each $y\in\Y$. Thus, by Theorem $18.19$ (pp. $605$) of \cite{aliprantis2006} (or Aumann's measurable selection principle \cite{steinwart2008,castaing2006}), we have that $m$ is measurable and that there exists a measurable $f^*:\X\to\Y$ such that $r(x, f^*(x)) = m(x)$ for all $x\in\X$. Moreover, by definition of $m$, given any measurable $f: \X \to \Y$, we have $m(x) \leq r(x, f(x))$. Therefore,
\eqals{
\E(f^*) = \int r(x, f^*(x))d\rhox(x) = \int m(x) d\rhox(x) \leq \int r(x, f(x)) d\rhox(x) = \E(f).
}
We conclude $\E(f^*) \leq \inf_{f:\X \to \Y} \E(f)$ and, since $f^*$ is measurable, $\E(f^*) = \min_{f:\X\to\Y} \E(f)$ and $f^*$ is a global minimizer.
\end{proof}

We have an immediate Corollary to \autoref{lemma:solution-structured-risk}.

\begin{corollary}\label{cor:solution-surrogate-risk}
With the hypotheses of \autoref{lemma:solution-structured-risk}, let $\tilde{f}:\X\to\Y$ such that 
$$\tilde{f}(x) \in \argmin_{y\in\Y} ~ \int_{\Y} \loss(y,y') d\rho(y'|x)$$
 for almost every $x \in \domrho$. Then $\E(\tilde{f}) = \inf_{f:\X\to\Y} \E(f)$.
\end{corollary}

\begin{proof}
The result follows directly from \autoref{lemma:solution-structured-risk} by noting that $r(x,\tilde{f}(x)) = m(x)$ almost everywhere on $\domrho$. Hence, since $\domrho$ is equal to the support of $\rhox$ up to a set of measure zero, $\E(\tilde{f}) = \int_\X m(x) d\rhox(x) = \E(f^*) = \inf_{f}\E(f)$. 
\end{proof}

With the above basic notation and results, we can proceed to prove the results presented in this work.

\section{Surrogate Problem, Fisher Consistency and Comparison Inequality}

In this section we focus on the surrogate framework introduced in Sec.~\ref{sec:surrogate} and prove that it is {\it Fisher consistent} and that the {\it comparison inequality}. To do so, we will first characterizes the solution(s) of the surrogate expected risk minimization introduced at \eqref{eq:surrogate_ls}. We recall that in our setting, the surrogate risk was defined as the functional $\mathcal{R}(g) = \int_{\X\times\Y} \|\psi(y) - g^*(x)\|_\HY^2 d\rho(x,y)$, where $\psi:\Y\to\HY$ is continuous (by \autoref{assumption:main}). In the following, when $\psi$ is bounded, we will denote with $Q = \sup_{y\in\Y} \|\psi(y)\|_\HY$. Note that in most our results we will assume $\Y$ to be compact. In these settings we always have $Q = \max_{y\in\Y} \|\psi(y)\|_\HY$ by the continuity of $\psi$.

We start with a preliminary lemma necessary to prove \autoref{lm:relation_expected_risk} and \autoref{teo:surrogate}.
\begin{lemma}\label{lemma:surrogate-problem-sol}
Let $\HY$ a separable Hilbert space and $\psi:\Y\to\HY$ measurable and bounded. Then, the function $g^*:\X\to\HY$ such that
\eqal{\label{eq:g_average}
    g^*(x) = \int_\Y \psi(y) d\rho(y|x) \quad \forall x\in \domrho
}
and $g^*(x) = 0$ otherwise, belongs to $\LXH$ and is a minimizer of the surrogate expected risk at \eqref{eq:surrogate_ls}. Moreover, any minimizer of \eqref{eq:surrogate_ls} is equal to $g^*$ almost everywhere on the domain of $\rhox$.
\end{lemma}
\begin{proof}
By hypothesis, $\|\psi\|_{\HY}$ is measurable and bounded. Therefore, since $\rho(y|x)$ is a regular conditional probability, we have that $g^*$ is measurable on $\X$ (see for instance \cite{steinwart2008}). Moreover, the norm of $g^*$ is dominated by the constant function of value $Q$, thus $g^*$ is integrable on $\X$ with respect to $\rhox$ and in particular it is in $L^2(\X,\rhox,\HY)$ since $\rhox$ is a finite regular measure. Recall that since $\rho(y|x)$ is a regular conditional distribution, for any measurable $g:\X\to\HY$, the functional in \eqref{eq:surrogate_ls} can be written as
\eqals{
{\cal R}(g) = \int_{\X\times\Y} \|g(x) - \psi(y)\|_\HY^2 d\rho(x,y) = \int_{\X} \int_\Y \|g(x) - \psi(y)\|_\HY^2 d\rho(y|x)d\rhox(x).
}
Notice that $g^*(x) = \argmin_{\eta \in \HY} \int_\Y \|\eta - \psi(y)\|_\HY^2 d\rho(y|x)$ almost everywhere on $\domrho$. Indeed,
\eqals{
  \int_\Y \|\eta - \psi(y)\|_\HY^2 d\rho(y|x) & = \|\eta\|_\HY^2 - 2 \ip{\eta}{\left(\int_\Y \psi(y)d\rho(y|x)\right)} + \int_\Y \|\psi(y)\|_\HY^2 d\rho(y|x) \\
  & = \|\eta\|_\HY^2 - 2 \ip{\eta}{g^*(x)}_\HY + const.
}
for all $x\in\domrho$, which is minimized by $\eta = g^*(x)$ for all $x\in\domrho$. Therefore, since $\domrho$ is equal to the support of $\rhox$ up to a set of measure zero, we conclude that $\mathcal{R}(g^*)\leq \inf_{g:\X\to\HY}\mathcal{R}(g)$ and, since $g^*$ is measurable, $\mathcal{R}(g^*) = \min_{g:\X\to\HY}\mathcal{R}(g)$ and $g^*$ is a global minimizer as required.

Finally, notice that for any $g:\X\to\HY$ we have
\eqal{\label{eq:equation_excess_ls_risk}
  \mathcal{R}(g) - \mathcal{R}(g^*) & = \int_{\X\times\Y} \|g(x) - \psi(y)\|_\HY^2 - \|g^*(x) - \psi(y)\|_\HY^2 d\rho(x,y) \\ 
  & = \int_{\X} \|g(x)\|_\HY^2 - 2 \ip{g(x)}{\left(\int_\Y \psi(y) d\rho(y|x)\right)}_\HY + \|g^*(x)\|_\HY^2 d\rhox(x) \\
  & = \int_\X \|g(x)\|_\HY^2 - 2 \ip{g(x)}{g^*(x)}_\HY + \|g^*(x)\|_\HY^2d\rhox(x) \\
  & = \int_\X \|g(x) - g^*(x)\|_\HY^2 d\rhox(x)
}
Therefore, for any measurable minimizer $g':\X\to\HY$ of the surrogate expected risk at \eqref{eq:surrogate_ls}, we have $\mathcal{R}(g')-\mathcal{R}(g^*) = 0$ which, by the relation above, implies $g'(x) = g^*(x)$ a.e. on $\domrho$.
\end{proof}
\Lrelation*
\begin{proof}
By \autoref{lemma:surrogate-problem-sol} we know that $g^*(x) = \int_\Y \psi(y)d\rho(y|x)$ almost everywhere on $\domrho$ and is the minimizer of $\cal{R}$. Therefore we have
\eqals{
  \ip{\psi(y)}{Vg^*(x)}_\HY & = \ip{\psi(y)}{V\int_\Y \psi(y')d\rho(y'|x)}_\HY \\
  & = \int_\Y \ip{\psi(y)}{V\psi(y')}_\HY d\rho(y'|x) = \int_\Y \loss(y,y')d\rho(y'|x)
}
for almost every $x\in\domrho$. Thus, for any measurable function $f: \X \to \Y$ we have
\eqals{
{\cal E}(f) &= \int_{\X\times\Y} \loss(f(x),y) d\rho(x,y) = \int_{X} \int_{\Y} \loss(f(x),y) d\rho(y|x)d\rhox(x) \\
&= \int_{X}\ip{\psi(f(x))}{Vg^*(x)}_\HY d\rhox(x).
}
\end{proof}

\Tsurrogate*
\begin{proof}
For the sake of clarity, the result for the fisher consistency and the comparison inequality are proven respectively in \autoref{proposition:fisher_consistency}, \autoref{teo:comparison_inequality}. The two results are proven below.
\end{proof}
 
\begin{restatable}[Fisher Consistency]{theorem}{Prelation}\label{proposition:fisher_consistency}
Let $\loss:\Y\times\Y\to\R$ satisfy \autoref{assumption:main} with $\Y$ a compact set. Let $g^*:\X\to\HY$ be a minimizer of the surrogate problem at \eqref{eq:surrogate_ls}. Then, for any decoding $\decoding:\HY\to\Y$ satisfying \eqref{eq:decoding_function}
\eqal{\label{eq:fisher}
\E(\decoding \circ g^*) = \inf_{f:\X\to\Y} \E(f)
}
\end{restatable}

\begin{proof}
It is sufficient to show that $\decoding \circ g^*$ satisfies \eqref{eq:solution_expected_risk} almost everywhere on $\domrho$. Indeed, by directly applying \autoref{cor:solution-surrogate-risk} we have $\E(\decoding \circ g^*) = \E(f^*) = \inf_f\E(f)$ as required.

We recall that a mapping $d:\HY\to\Y$ is a decoding for our surrogate framework if it satisfies \eqref{eq:decoding_function}, namely
\eqals{
  \decoding(\eta) \in \argmin_{y\in\Y} \  \ip{\psi(y)}{V\eta}_\HY \mbox{ \ \ \ \ \ } \forall \eta\in\HY.
}
By \autoref{lemma:surrogate-problem-sol} we know that $g^*(x) = \int_\Y \psi(y)d\rho(y|x)$ almost everywhere on $\domrho$. Therefore, we have
\eqals{
  \ip{\psi(y)}{Vg^*(x)}_\HY & = \ip{\psi(y)}{V\int_\Y \psi(y')d\rho(y'|x)}_\HY \\
  & = \int_\Y \ip{\psi(y)}{V\psi(y')}_\HY d\rho(y'|x) = \int_\Y \loss(y,y')d\rho(y'|x)
}
for almost every $x\in\domrho$. As a consequence, for any $\decoding:\HY\to\Y$ satisfying \eqref{eq:decoding_function}, we have
\eqals{
d \circ g^*(x) \in \argmin_{y \in \Y} \ip{\psi(y)}{V g^*(x)}_\HY = \argmin_{y \in \Y} \int_\Y \loss(y, y') d\rho(y'|x)
}
almost everywhere on $\domrho$. We are therefore in the hypotheses of \autoref{cor:solution-surrogate-risk} with $\tilde{f} = \decoding \circ g^*$, as desired.
\end{proof}
The Fisher consistency of the surrogate problem allows to prove the comparison inequality (\autoref{teo:comparison_inequality}) between the excess risk of the structured prediction problem, namely $\E(\decoding \circ g) - \E(f^*)$, and the excess risk $\mathcal{R}(g) - \mathcal{R}(g^*)$ of the surrogate problem. However, before showing such relation, in the following result we prove that for any measurable $g:\X\to\HY$ and measurable decoding $\decoding:\HY\to\Y$, the expected risk $\E(\decoding \circ g)$ is well defined.
\begin{lemma}\label{lm:exp-risk-well-def}
Let $\Y$ be compact and $\loss:\Y\times\Y\to\R$ satisfying \autoref{assumption:main}. Let $g:\X\to\HY$ be measurable and $\decoding:\HY\to\Y$ a measurable decoding satisfying \eqref{eq:decoding_function}. Then $\E(\decoding \circ g)$ is well defined and moreover $|\E(\decoding \circ g)|\leq \Q^2 \|V\|$.
\end{lemma}
\begin{proof}
$\loss(\decoding \circ g(x),y)$ is measurable in both $x$ and $y$ since $\loss$ is continuous and $\decoding \circ g$ is measurable by hypothesis (combination of measurable functions). Now, $\loss$ is pointwise bounded by $\Q^2 \|V\|$ since
\eqal{
    |\loss(y,y')| = |\ip{\psi(y)}{V\psi(y')}_\HY|\leq\|\psi(y)\|_\HY^2 \|V\| \leq \Q^2\|V\|.
}
Hence, by Theorem $11.23$ pp. $416$ in \cite{aliprantis2006} the integral of $\loss(\decoding \circ g(x),y)$ exists and therefore
\eqal{
    |\E(\decoding \circ g)| \leq \int_\X |\loss(\decoding \circ g(x),y)| d\rho(x,y) \leq \Q^2 \|V\| < +\infty.
}
\end{proof}
A question introduced by \autoref{lm:exp-risk-well-def} is whether a {\it measurable} decoding always exists. The following result guarantees that, under the hypotheses introduced in this work, a decoding $\decoding:\HY\to\Y$ satisfying \eqref{eq:decoding_function} always exists.
\begin{lemma}
Let $\Y$ be compact and $\psi:\Y\to\HY$ and $V:\HY\to\HY$ satisfy the requirements in \autoref{assumption:main}. Define $m:\X\to\R$ as 
\begin{equation}\label{eq:compute_m}
m(\eta) = \min_{y\in\Y} \ \ip{\psi(y)}{V\eta}_\HY \mbox{ \ \ } \forall y\in\Y, \eta\in\HY.
\end{equation}
Then, $m$ is measurable and there exists a measurable decoding $\decoding:\HY\to\Y$ satisfying \eqref{eq:decoding_function}, namely such that $m(\eta) = \ip{\psi(\decoding(\eta))}{V\eta}_\HY$ for each $\eta\in\HY$.
\end{lemma}
\begin{proof}
Similarly to the proof of \autoref{lemma:solution-structured-risk}, the result is a direct application of Theorem $18.19$ (pp. $605$) of \cite{aliprantis2006} (or Aumann's measurable selection principle \cite{steinwart2008,castaing2006}).
\end{proof}
We now prove the {\em comparison inequality} at \eqref{eq:comparison_inequality_ls}. 
\begin{restatable}[Comparison Inequality]{theorem}{Tcomparison}\label{teo:comparison_inequality}
Let $\loss:\Y\times\Y\to\R$ satisfy \autoref{assumption:main} with $\Y$ a compact or finite set. Let $f^*:\X\to\Y$ and $g^*:\X\to\HY$ be respectively solutions to the structured and surrogate learning problems at \eqref{eq:expected_risk_minimization} and \eqref{eq:surrogate_ls}. Then, for every measurable $g:\X\to\HY$ and $d:\HY\to\Y$ satisfying \eqref{eq:decoding_function}
\begin{equation}
    \E(\decoding \circ g) - \E(f^*) \leq 2 \Q \|V\| \sqrt{\mathcal{R}(g) - \mathcal{R}(g^*)}.
\end{equation}
\end{restatable}
\begin{proof}
Let us denote $f = \decoding \circ g$ and $f_0 = \decoding \circ g^*$. By  Thm.~\ref{proposition:fisher_consistency} we have that $\E(f_0) = \inf_{f:\X \to \Y}\E(f)$ and so $\E(f^*) = \E(f_0)$. Now, by combining \autoref{assumption:main} with \autoref{lemma:surrogate-problem-sol}, we have
\eqals{
    \E(f) - \E(f^*) &= \E(f) - \E(f_0) = \int_{\X\times\Y} \loss(f(x),y) - \loss(f_0(x),y) d\rho(x,y) \\ 
                    & = \int_{\X\times\Y} \ip{\psi(f(x)) - \psi(f_0(x))}{V\psi(y)}_\HY d\rho(x,y) \\ 
                    & = \int_\X \ip{\psi(f(x)) - \psi(f_0(x))}{V\left(\int_\Y \psi(y)d\rho(y|x)\right)}_\HY d\rhox(x) \\
                    & = \int_\X \ip{\psi(f(x)) - \psi(f_0(x))}{Vg^*(x)}_\HY d\rhox(x) \\ 
                    & = A + B.
}
where 
\eqals{
    A & = \int_\X \ip{\psi(f(x))}{V(g^*(x) - g(x))}_\HY  ~d\rhox(x) \\
    B &= \int_\X \ip{\psi(f(x))}{Vg(x)}_\HY  d\rhox(x) - \int_\X  \ip{\psi(f_0(x))}{Vg^*(x)}_\HY d\rhox(x) 
}
Now, the term A can be minimized by taking the supremum over $\Y$ so that
\eqals{
    A \leq \int_\X \sup_{y\in\Y} \Big|\ip{\psi(y)}{V(g^*(x) - g(x))}_\HY\Big| d\rho_\X(x).
}
For B, we observe that, by the definition of the decoding $d$, we have
\eqals{
\ip{\psi(f_0(x))}{Vg^*(x)}_\HY  &= \inf_{y' \in \Y} \ip{\psi(y')}{Vg^*(x)}_\HY, \\ \ip{\psi(f(x))}{Vg(x)}_\HY  &= \inf_{y' \in \Y} \ip{\psi(y')}{Vg(x)}_\HY,
}
for all $x\in\X$l. Therefore,
\begin{align}
    B   & = \int_\X ~ \inf_{y\in\Y} \ip{\psi(y)}{Vg(x)}_\HY - \inf_{y\in\Y} \ip{\psi(y)}{Vg^*(x)}_\HY ~ d\rho_\X(x) \\
        & \leq \int_\X \sup_{y\in\Y} \Big|\ip{\psi(y)}{V(g(x)- g^*(x))}_\HY\Big| d\rho_\X(x) 
\end{align}
where we have used the fact that for any given two functions $\eta,\zeta:\Y\to\R$ we have 
\begin{equation}
|\inf_{y\in\Y} \eta(y) - \inf_{y\in\Y} \zeta(y)| \leq \sup_{y\in\Y} |\eta(y) - \zeta(y)|.
\end{equation}
Therefore, by combining the bounds on $A$ and $B$ we have
\begin{align*}
    \E(f) - \E(f^*) & \leq 2 \int_\X \sup_{y\in\Y} \Big|\ip{\psi(y)}{V(g^*(x) - g(x))}_\HY\Big| d\rho_\X(x) \\
    & \leq 2 \int_\X \sup_{y\in\Y} \|V^*\psi(y)\|_\HY \|g^*(x) - g(x)\|_\HY d\rho_\X(x) \\
    & \leq 2 \Q \|V\| \int_\X \|g^*(x) - g(x)\|_\HY d\rho_\X(x) \\
    & \leq 2 \Q \|V\| \sqrt{\int_\X \|g^*(x) - g(x)\|_\HY^2 d\rho_\X(x)}, \\
\end{align*}
where for the last inequality we have used the Jensen's inequality. The proof is concluded by recalling that (see \eqref{eq:equation_excess_ls_risk})
\begin{equation}\label{eq:residual_error_equivalence}
    \mathcal{R}(g) - \mathcal{R}(g^*) = \int_{\X} \|g(x) - g^*(x)\|_\HY^2 d\rho_\X(x)
\end{equation}
\end{proof}

\section{Learning Bounds for Structured Prediction}

In this section we focus on the analysis of the structured prediction algorithm proposed in this work (\ref{eq:algorithm}). In particular, we will first prove that, given the minimizer $\hat{g}:\X\to\HY$ of the empirical risk at \eqref{eq:kde}, its decoding can be computed in practice according to \ref{eq:algorithm}. Then, we report the proofs for the universal consistency of such approach (\autoref{teo:universal_consistency}) and generalization bounds (\autoref{teo:simple_bound}).

\subsection*{Notation}

Let $k:\X\times\X\to\R$ a positive semidefinite function on $\X$, we denote $\HX$ the Hilbert space obtained by the completion
\eqals{
  \HX = \overline{\Span \{ k(x,\cdot) \ | \ x\in\X\} }
}
according to the norm induced by the inner product $\ip{k(x,\cdot)}{k(x',\cdot)}_\HX = k(x,x')$. Spaces $\HX$ constructed in this way are known as {\it reproducing kernel Hilbert spaces} and there is a one-to-one relation between a kernel $k$ and its associated RKHS. For more details on RKHS we refer the reader to \cite{berlinet2011}. Given a kernel $k$, in the following we will denote with $\varphi:\X\to\HX$ the feature map $\varphi(x) = k(x,\cdot) \in\HX$ for all $x\in\X$. We say that a kernel is bounded if $\|\varphi(x)\|_\HX \leq \kappa$ with $\kappa>0$. Note that $k$ is bounded if and only if $k(x,x') = \ip{\varphi(x)}{\varphi(x')}_\HX \leq \|\varphi(x)\|_\HX \|\varphi(x')\|\leq \kappa^2$ for every $x,x'\in\X$.  In the following we will always assume $k$ to be continuous and bounded by $\kappa>0$. The continuity of $k$ with the fact that $\X$ is Polish implies $\HX$ to be separable \cite{berlinet2011}.

We introduce here the ideal and empirical operators that we will use in the following to prove the main results of this work. 

\begin{itemize}
\item $S:\HX\to L^2(\X,\rho_\X)$ s.t. $f\in\HX\mapsto\ip{f}{\varphi(\cdot)}_\HX\in L^2(\X,\rho_\X)$, with adjoint 
\item $S^*:\LX\to\HX$ s.t. $h\in\LX\mapsto \int_\X h(x)\varphi(x)d\rhox(x)\in\HX$,
\item $Z:\HY\to L^2(\X,\rho_\X)$ s.t. $h\in\HY\mapsto\ip{h}{g^*(\cdot)}_\HY\in L^2(\X,\rhox)$, with adjoint
\item $Z^*:\LX\to\HY$ s.t. $h\in\LX\mapsto \int_\X h(x)g^*(x)d\rhox(x)\in\HY$,
\item $C = S^*S:\HX\to\HX$ and $L = SS^*:L^2(\X,\rho_\X)\to L^2(\X,\rho_\X)$,
\end{itemize}
with $g^*(x) = \int_\Y \psi(y)d\rho(y|x)$ defined according to \eqref{eq:g_average}, (see \autoref{lemma:surrogate-problem-sol}). 

Given a set of input-output pairs $\{(x_i,y_i)\}_{i=1}^n$ with $(x_i,y_i)\in\X\times\Y$ independently sampled according to $\rho$ on $\X\times\Y$, we define the empirical counterparts of the operators just defined as
\begin{itemize}
\item $\hat{S}:\HX\to\R^n$ s.t. $f \in\HX \mapsto \frac{1}{\sqrt{n}}(\ip{\varphi(x_i)}{f}_\HX)_{i=1}^n \in \R^n$, with adjoint
\item $\hat{S}^*:\R^n\to\HX$ s.t. $v = (v_i)_{i=1}^n\in\R^n \mapsto \frac{1}{\sqrt{n}} \sum_{i=1}^n v_i \varphi(x_i)$,
\item $\hat{Z}:\HY\to\R^n$ s.t. $h \in\HY \mapsto \frac{1}{\sqrt{n}}(\ip{\psi(y_i)}{h}_\HY)_{i=1}^n \in \R^n$, with adjoint
\item $\hat{Z}^*:\R^n\to\HY$ s.t. $v = (v_i)_{i=1}^n\in\R^n \mapsto \frac{1}{\sqrt{n}} \sum_{i=1}^n v_i \psi(y_i)$,
\item $\hat{C} = \hat{S}^*\hat{S}:\HX\to\HX$ and $K = n \hat{S}\hat{S}^*\in\R^{n \times n}$ is the empirical kernel matrix. 
\end{itemize}
In the rest of this section we denote with $A + \la$, the operator $A + \la I$, for any symmetric linear operator $A$, $\la \in \R$ and $I$ the identity operator.

We recall here a basic result characterizing the operators introduced above.

\begin{proposition}\label{prop:basic_operator_result}
With the notation introduced above, 
\eqal{
  C = \int_\X \varphi(x) \otimes \varphi(x) d\rhox(x) \mbox{ \ \ \ \ and \ \ \ \ } Z^*S = \int_{\X\times\Y} \psi(y) \otimes \varphi(x) d\rho(x,y)
}
where $\otimes$ denotes the tensor product. Moreover, when $\varphi$ and $\psi$ are bounded by respectively $\kappa$ and $\Q$, we have the following facts
\begin{enumerate}[(i)]
\item $\tr(L) = \tr(C) = \|S\|_{HS}^2 = \int_\X \|\varphi(x)\|_\HX^2 d\rhox(x) \leq \kappa^2$
\item $\|Z\|_{HS}^2 = \int_X \|g^*(x)\|^2 d\rhox(x) = \|g^*\|_\rhox^2 < +\infty$.
\end{enumerate}
\end{proposition}
\begin{proof}
By definition of $C = S^*S$, for each $h,h'\in\HX$ we have
\eqals{
  \ip{h}{Ch'}_\HX = \ip{Sh}{Sh'}_\rhox  & = \int_\X \ip{h}{\varphi(x)}_\HX \ip{\varphi(x)}{h'}_\HX d\rhox(x) \\ 
  & = \int_\X \left\ip{h}{\Big(\varphi(x)\ip{\varphi(x)}{h'}_\HX\Big)\right}_\HX d\rhox(x) \\
  & = \int_\X \left\ip{h}{\Big(\varphi(x)\otimes\varphi(x) \Big)h'\right} d\rhox(x) \\
  & = \ip{h}{\Big(\int_\X \varphi(x)\otimes\varphi(x)d\rhox(x)\Big)h'}_\HX
}
since $\varphi(x)\otimes\varphi(x):\HX\to\HX$ is the operator such that $h\in\HX\mapsto \varphi(x)\ip{\varphi(x)}{h}_\HX$. The characterization for $Z^*S$ is analogous.

Now, $(i)$. The relation $\tr(L) = \tr(C) = \tr(S^*S) = \|S\|_{HS}^2$ holds by definition. Moreover
\eqals{
  \tr(C) = \int_\X \tr(\varphi(x) \otimes \varphi(x)) d\rhox(x) = \int_\X \|\varphi(x)\|_\HX^2 d\rhox(x)
}
by linearity of the trace. $(ii)$ is analogous. Note that $\|g^*\|_\rhox^2 < +\infty$. by \autoref{lemma:surrogate-problem-sol} since $\psi$ is bounded by hypothesis.

\end{proof}

\subsection{Reproducing Kernel Hilbert Spaces for Vector-valued Functions}

We begin our analysis by introducing the concept of reproducing kernel Hilbert space (RKHS) for vector-valued functions. Here we provide a brief summary of the main properties that will be useful in the following. We refer the reader to \cite{micchelli2004,carmeli2006} for a more in-depth introduction on the topic.

Analogously to the case of scalar functions, a RKHS for vector-valued functions $g:\X\to\H$, with $\H$ a separable Hilbert space, is uniquely characterized by a so-called {\it kernel of positive type}, which is an operator-valued $\G:\X\times\X\to\B(\H,\H)$ generalizing the concept of scalar reproducing kernel.

\begin{definition}
Let $\X$ be a set and $\H$ be a Hilbert space, then $\G:\X\times\X\to\B(\H,\H)$ is a {\it kernel of positive type} if for each $n\in\N$, $x_1,\dots,x_n\in\X$, $c_1,\dots,c_n\in\H$ we have
\begin{equation}
    \sum_{i,j=1}^n \ip{\G(x_i,x_j) c_i }{c_j}_\H \geq 0
\end{equation}
\end{definition}

A kernel of positive type $\G$ defines an inner product $\ip{\G(x,\cdot)c}{\G(x',\cdot)c'}_{\mG_0} = \ip{\G(x,x')c}{c'}_\H$ on the space
\begin{equation}\label{eq:construction_of_RKHSvv}
    \mG_0 = \Span\{\G(x,\cdot)c \ | \ x\in\X, c\in\H \}.
\end{equation}
Then, the completion $\mG = \overline{\mG_0}$ with respect to the norm induced by $\ipv_{\mG_0}$ is known as the reproducing Kernel Hilbert space (RKHS) for vector-valued functions associated to the kernel $\G$. Indeed, we have that a reproducing property holds also for RKHS of vector-valued functions, namely for any $x\in\X$, $c\in\H$ and $g\in\mG$ we have\begin{equation}\label{eq:reproducing_property}
    \ip{g(x)}{c}_\H = \ip{g}{\G(x,\cdot)c}_\mG
\end{equation}
and that for each $x\in\X$ the function $\G(x,\cdot):\mG\to\H$ is the evaluation functional in $x$ on $\mG$, namely $\G(x,\cdot)(g) = g(x)$.

\subsubsection{Separable Vector Valued Kernels}

In this work we restrict to the special case of RKHS for vector-valued functions with associated kernel $\G:\X\times\X\to\HY$ of the form $\G(x,x') = k(x,x') I_\HY$ for each $x,x'\in\X$, where $k:\X\times\X\to\R$ is a scalar reproducing kernel and $I_\HY$ is the identity operator on $\HY$.  Notice that this choice is not restrictive in terms of the space of functions that can be learned by our algorithm. Indeed, it was proven in \cite{carmeli2010} (see Example 14) that if $k$ is a universal scalar kernel, then $\G(\cdot,\cdot) = k(\cdot,\cdot) I_\HY$ is universal. Below, we report a useful characterization of RKHS $\mG$ associated to a separable kernel.

\begin{restatable}{lemma}{Lhxy}\label{lemma:hxy_bis}
The RKHS $\mG$ associated to the kernel $\G(x,x') = k(x,x')I_\HY$ is isometric to $\HYX$ and for each $g\in\mG$ there exists a unique $G\in\HYX$ such that 
\begin{equation}\label{eq:evaluation_HS}
    g(x) = G \varphi(x) \in\HY \mbox{ \ \ \ for each \ } x\in\X
\end{equation}
\end{restatable}
\begin{proof}
We explicitly define the isometry $T:\mG\to\HYX$ as the linear operator such that $T(\sum_{i=1}^n \alpha_i \G(x_i,\cdot)c_i) = \sum_{i=1}^n \alpha_i \ c_i \otimes \varphi(x_i)$ for each $n\in\N$, $x_1,\dots,x_n\in\X$ and $c_1,\dots,c_n\in\HY$. By construction, $T(\mG_0) \subseteq \HYX$, with $\mG_0$ the linear space defined at \eqref{eq:construction_of_RKHSvv} and moreover  
\eqal{
\ip{T(\G(x,\cdot)c)}{T(\G(x',\cdot)c')}_{HS}  = \ip{c\otimes\varphi(x)^*}{c'\otimes\varphi(x')^*}_{HS} = \ip{c}{c'}_\HY \ip{\varphi(x)}{\varphi(x')}_\HX \\ 
\quad = \ip{k(x,x') c}{c'}_\HY  = \ip{\G(x,x')c}{c'}_\HY = \ip{\G(x,\cdot)c}{\G(x',\cdot)c'}_\mG
}
implying that $\mG_0$ is isometrically contained in $\HYX$. Since $\HYX$ is complete, also $\overline{T(\mG_0)}\subseteq\HYX$. Therefore $\mG$ is isometrically contained in $\HYX$, since $T(\mG) = T(\overline{\mG_0}) = \overline{T(\mG_0)}$. Moreover note that
\eqals{
    \overline{T(\mG_0)} & = \overline{\Span\{c\otimes\varphi(x) \ | \ x\in\X, c\in\HY \}} = \overline{\Span\{c \ | \ c\in\HY \} \otimes \Span\{\varphi(x) \ | \ x\in\X \}} \\
    & = \overline{\HY \otimes \Span\{\varphi(x) \ | \ x\in\X\} } = \overline{\HY} \otimes \overline{\Span\{\varphi(x) \ | \ x\in\X\} } = \HYX
}
from which we conclude that $\mG$ is isometric to $\HYX$ via $T$.

To prove \eqref{eq:evaluation_HS}, let us consider $x\in\X$ and $g\in\mG$ with $G = T(g) \in\HYX$. Then, $\forall c\in\HY$ we have that 
\begin{align}
    \ip{c}{g(x)}_\HY & = \ip{\G(x,\cdot)c}{g}_\mG = \ip{T(\G(x,\cdot)c)}{G}_{HS} \\ 
    & = \ip{c\ot\varphi(x)}{G}_{HS} = \tr((\varphi(x) \otimes c)G) = \tr(c^*G\varphi(x)) = \ip{c}{G\varphi(x)}_{\HY}.
\end{align}
Since the equation above is true for each $c\in\HY$ we can conclude that $g(x) = T(g)\varphi(x)$ as desired.
\end{proof}

The isometry $\mG \simeq \HYX$ allows to characterize the closed form solution for the surrogate risk introduced in \eqref{eq:surrogate_ls}. We recall that in \autoref{lemma:surrogate-problem-sol}, we have shown that $\mathcal{R}$ always attains a minimizer on $\LXH$. In the following we show that if $\mathcal{R}$ attains a minimum on $\mG \simeq \HYX$, we are able to provide a close form solution for one such element.

\begin{lemma}\label{lemma:expected_risk_solution}
Let $\psi$ and $\HY$ satisfying \autoref{assumption:main} and assume that the surrogate expected risk minimization of ${\cal R}$ at \eqref{eq:surrogate_ls} attains a minimum on $\mG$, with $\mG \simeq\HYX$. Then the minimizer $g^* \in \mG$ of $\mathcal{R}$ with minimal norm $\|\cdot\|_\mG$ is of the form
\begin{equation}
    g^*(x) = G\varphi(x),\quad \forall x \in \X \quad \mbox{with} \quad G = Z^*SC^\dagger \in \HYX.
\end{equation}
\end{lemma}

\begin{proof} 
By hypothesis we have $g^*\in\mG$. Therefore, by applying Lemma~\ref{lemma:hxy_bis} we have that there exists a unique linear operator $G: \HX \to \HY$ such that $g^*(x) = G\varphi(x), \forall x \in \X$. Now, expanding the least squares loss on $\HY$, we obtain 
\eqal{
    \mathcal{R}(g) & = \int_{\X\times\Y} \|G\varphi(x) - \psi(y)\|_\HY^2 d\rho(x,y) \\ 
  & = \tr(G(\varphi(x)\ot\varphi(x))G^*) d\rhox(x,y) - 2 \tr(G(\varphi(x)\ot\psi(y))) + \int_{\X\times\Y} \|\psi(y)\|_\HY^2  \\
  & = \tr(GCG^*) - 2\tr(GS^*Z) + const.
}
where we have used \autoref{prop:basic_operator_result} and the linearity of the trace. Therefore $\mathcal{R}$ is a quadratic functional, and is convex since $C$ is positive semidefinite. We can conclude that $\mathcal{R}$ attains a minimum on $\mG$ if and only if the range of $S^*Z$ is contained in the range of $C$, namely $\ran(S^*Z) \subseteq \ran(C) \subset \HX$ (see \cite{engl1996regularization} Chap. 2). In this case $G = Z^*SC^\dagger \in \HYX$ exists and is the minimum norm minimizer for $\mathcal{R}$, as desired.
\end{proof}

Analogously to \autoref{lemma:expected_risk_solution}, a closed form solution exists for the regularized surrogate empirical risk minimization problem introduced in \eqref{eq:kde}. We recall that the associated functional is $\hat{\mathcal{R}}_\lambda:\mG\to\R$ defined as
\eqals{
\hat{\mathcal{R}_\lambda}(g) = \hat{\mathcal{R}}(g) + \lambda ~ \|g\|_\mG^2 = \frac{1}{n}\sum_{i=1}^n \|g(x_i) - \psi(y_i)\|_\HY^2 + \la ~ \|g\|_\HY^2
}
for all $g\in\mG$ and $\{(x_i,y_i)\}_{i=1}^n$ points in $\X\times\Y$. The following result characterizes the closed form solution for the empirical risk minimization when $\mG \simeq \HYX$ and guarantees that such a solution always exists.

\begin{lemma}\label{lemma:empirical_risk_solution}
Let $\mG  \simeq \HYX$. For any $\lambda>0$, the solution $\hat{g}_\lambda\in\mG$ of the empirical risk minimization problem at \eqref{eq:kde} exists, is unique and is such that
\eqal{\label{eq:empirical_closed_form}
  \hat{g}_\lambda(x) = \hat{G}_\lambda \varphi(x), \quad \forall x\in\X \quad \mbox{with} \quad \hat{G}_\lambda = \hat{Z}^*\hat{S}(\hat{C} + \lambda)^{-1} \in \HYX.
}
\end{lemma}

\begin{proof}
The proof of is analogous to that of \autoref{lemma:expected_risk_solution} and we omit it. Note that, since $(\hat{C} + \la)^{-1}$ is always bounded for $\la > 0$, its range corresponds to $\HX$ and therefore the range of $S^*Z$ is always contained in it. Then $\hat{G}_\lambda$ exists for any $\la > 0$ and is unique since $\|\cdot\|_\HY^2$ is strictly convex.
\end{proof}

The closed form solutions provided by \autoref{lemma:expected_risk_solution} and \autoref{lemma:empirical_risk_solution} will be key in the analysis of the structured prediction algorithm \ref{eq:algorithm} in the following.

\subsection{The Structured Prediction Algorithm}

In this section we prove that \ref{eq:algorithm} corresponds to the decoding of the surrogate empirical risk minimizer $\hat{g}$ (\eqref{eq:kde}) via a map $\decoding: \HY \to \Y$ satisfying \eqref{eq:decoding_function}.

Recall that in \autoref{lemma:hxy_bis} we proved that the vector-valued RKHS $\mG$ induced by a kernel $\G(x,x') = k(x,x') I_\HY$, for a scalar kernel $k$ on $\X$, is isometric to $\HYX$. For the sake of simplicity, in the following, with some abuse of notation, we will not make the distinction between $\mG$ and $\HYX$ when it is clear from context.

\Pcomputable*

\begin{proof}
From \autoref{lemma:empirical_risk_solution} we know that $\hat{g}(x) = \hat{Z}^*\hat{S}(\hat{C}+\la)^{-1}\varphi(x)$ for all $x\in\X$. Recall that $\hat{C} = \hat{S}^*\hat{S}$ and $\K = n\hat{S}\hat{S}^*\in\R^{n \times n}$, is the empirical kernel matrix associated to the inputs, namely such that $\K_{ij} = k(x_i,x_j)$ for each $i,j=1,\dots,n$. Therefore we have $\hat{S}(\hat{C} + \la)^{-1} = \sqrt{n}(\K+\la n)^{-1}\hat{S}:\HX\to\R^n$. Now, by denoting $\Kx = \sqrt{n}\hat{S}\varphi(x) = (k(x_i,x))_{i=1}^n\in\R^n$, we have 
\eqals{
\hat{S}(\hat{C}+\la)^{-1}\varphi(x) = (\K + \la n I)^{-1}\Kx = \alpha(x) \in \R^n.
}
Therefore, by applying the definition of the operator $\hat{Z}:\HY\to\R^n$, we have 
\eqal{
    \hat{g}(x) = \hat{Z}^* \hat{S}(\hat{C}+\la)^{-1}\varphi(x) = \hat{Z}^* \alpha(x) = \sum_{i=1}^n \alpha_i(x) \psi(y_i), \quad \forall x  \in \X
}
By plugging $\hat{g}(x)$ in the functional minimized by the decoding (\eqref{eq:decoding_function}),
\eqal{
    \ip{\psi(y)}{V\hat{g}(x)}_\HY = \ip{\psi(y)}{& V\hat{Z}^*\alpha(x)}_\HY = \ip{\psi(y)}{\sum_{i=1}^n \alpha_i(x)\psi(y_i)}_\HY \\
    & = \sum_{i=1}^n \alpha_i(x) \ip{\psi(y)}{V\psi(y_i)}_\HY = \sum_{i=1}^n \alpha_i(x) \loss(y,y_i),
}
where we have used the bilinear form for $\loss$ in \autoref{assumption:main} for the final equation. We conclude that
\eqal{
    d\circ \hat{g}(x) \in \argmin_{y\in\Y} \ip{\psi(y)}{V\hat{g}(x)}_\HY = \argmin_{y\in\Y} \sum_{i=1}^n \alpha_i(x) \loss(y,y_i)
}
as required.
\end{proof}

\autoref{prop:computable} focuses on the computational aspects of \ref{eq:algorithm}. In the following we will analyze its statistical properties.

\subsection{Universal Consistency}\label{sec:appendix_bounds}

In following, we provide a probabilistic bound on the excess risk $\mathcal{E}(\decoding \circ g) - \mathcal{E}(f^*)$ for any $g\in\mG~\simeq \HYX$ that will be key to prove both universal consistency (\autoref{teo:universal_consistency}) and generalization bounds (\autoref{teo:simple_bound}). To do so, we will make use of the comparison inequality from \autoref{teo:comparison_inequality} to control the structured excess risk by means of the excess risk of the surrogate $\mathcal{R}(g) - \mathcal{R}(g^*)$. Note that the surrogate problem consists in a vector-valued kernel ridge regression estimation. In this setting, the problem of finding a probabilistic bound has been studied (see \cite{caponnetto2007} and references therein). Indeed, our proof will consist of a decomposition of the surrogate excess risk that is similar to the one in \cite{caponnetto2007}. However, note that we cannot direct apply \cite{caponnetto2007} to our setting, since in \cite{caponnetto2007} the operator-valued kernel $\G$ associated to $\mG$ is required to be such that $\G(x,x')$ is trace class $\forall x,x'\in\X$, which does not hold for the kernel used in this work, namely $\G(x,x') = k(x,x')I_\HY$ when $\HY$ is infinite dimensional.

In order to express the bound on the excess risk more compactly, here we introduce a measure for the approximation error of the surrogate problem. According to \cite{caponnetto2007}, we define the following quantity
\eqal{\label{eq:Alambda}
{\cal A}(\la) = \la \| Z^* (L + \la)^{-1} \|_{HS}.
}
% Notice that by definition of $Z$, ${\cal A}(\la)$ is such that ${\cal A}(\la) = \la (\int \|q(x)\|_\HY^2 d\rhox(x))$ for all $\la>0$, with $q :\X \to \HY$ the solution of the integral equation 
% \eqals{
% \int k(x,x')\ip{h}{q(x')}_\HY d\rhox(x') + \la \ip{h}{q(x)}_\HY = \ip{h}{g^*(x)}_\HY,
% }
% almost surely on $x \in \X$ and for any $h \in \HY$.
\begin{lemma}\label{lemma:prob-bound}
Let $\Y$ be compact, $\loss:\Y\times\Y\to\R$ satisfying \autoref{assumption:main} and $k$ a bounded positive definite kernel on $\X$ with $\sup_{x\in\X}k(x,x) \leq \kappa^2$ and associated RKHS $\HX$. Let $\rho$ a Borel probability measure on $\X\times\Y$ and $\{(x_i,y_i)\}_{i=1}^n$ independently sampled according to $\rho$. Let $f^*$ be a solution of the problem in \eqref{eq:expected_risk_minimization}, $\hat{f} = \decoding \circ \hat{g}$ as in \ref{eq:algorithm}. Then, for any $\la \leq \kappa^2$ and $\delta > 0$, the following holds with probability $1 - \delta$:
\eqals{
\E(\hat{f}) - \E(f^*) \leq 8 \kappa Q \|V\| \frac{Q + {\cal B}(\la)}{\sqrt{\la n}}\left(1 + \sqrt{\frac{4\kappa^2}{\la \sqrt{n}}}\right)\log^2 \frac{8}{\delta} + 2 Q \|V\|{\cal A}(\la),
}
with ${\cal B}(\la) = \kappa\|Z^* S (C +\la)^{-1}\|_{HS}$.
\end{lemma}
\begin{proof}
According to \autoref{teo:comparison_inequality}
\eqal{
\E(d \circ \hat{g}) -  \E(f^*) \leq 2 Q \|V\|\sqrt{{\cal R}(\hat{g}) - {\cal R}(g^*)}.
}
From \autoref{lemma:empirical_risk_solution} we know that $\hat{g}(x) = \hat{G}\varphi(x)$ for all $x\in\X$, with $\hat{G} = \hat{Z}^*\hat{S}(\hat{C}+\la)^{-1} \in \HYX$. By \autoref{lemma:surrogate-problem-sol}, we know that $g^*(x) =  \int_\Y \psi(y) d\rho(y|x)$ almost everywhere on the support of $\rhox$. Therefore, a direct application of \autoref{prop:basic_operator_result} leads to
\eqals{
  {\cal R} & (\hat{g}) - {\cal R}(g^*) = \int \|\hat{g}(x) - g^*(x)\|_\HY^2 d\rhox(x) = \\
  & = \int_\X \|\hat{G}\varphi(x)\|_\HY^2 - 2 \ip{\hat{G}\varphi(x)}{g^*(x)}_\HY + \|g^*(x)\|_\HY^2 d\rhox(x) \\
  & = \int_\X \tr\left(\hat{G} \Big( \varphi(x) \otimes \varphi(x) \Big) \hat{G}^*\right) - 2 \tr\left(\hat{G} \Big(\varphi(x)\otimes g^*(x)\Big)\right) + \tr(g^*(x) \otimes g^*(x)) d\rhox(x) \\
  & = \tr(\hat{G}S^*S\hat{G}) - 2\tr(\hat{G}S^*Z) + \tr(Z^*Z) = \|\hat{G}S^* - Z^*\|_{HS}^2
}  
To bound $\|\hat{G}S^* - Z^*\|_{HS}$, we proceed with a decomposition similar to the one in \cite{caponnetto2007}. In particular $\|\hat{G}S^* - Z^*\|_{HS} \leq A_1 + A_2 + A_3$, with
\eqals{
  A_1 &= \|\hat{Z}^*\hat{S}(\hat{C}+\la)^{-1}S^* - Z^* S (\hat{C}+\la)^{-1} S^* \|_{HS} \\
  A_2 &= \|Z^* S (\hat{C}+\la)^{-1} S^* - Z^* S (C+\la)^{-1} S^* \|_{HS} \\
  A_3 &= \|Z^* S (C+\la)^{-1} S^* - Z^* \|_{HS}.
}
Let $\tau = \delta/4$.
Now, for the term $A_1$, we have
\eqals{
  A_1 & \leq \|\hat{Z}^*\hat{S} - Z^* S\|_{HS} \| (\hat C + \la)^{-1} S^*\|.
}
To control the term $\|\hat{Z}^*\hat{S} - Z^* S\|_{HS}$, note that $\hat{Z}^*\hat{S} = \frac{1}{n} \sum_{i=1}^{n}\zeta_i$ with $\zeta_i$ the random variable $\zeta_i = \psi(y_i) \otimes \varphi(x_i)$. By \autoref{prop:basic_operator_result}, for any $1 \leq i \leq n$ we have
\eqals{
\mathbb{E} \zeta_i = \int \psi(y) \otimes \varphi(x) d\rho(x,y) = Z^*S,
}
and
\eqals{
 \|\zeta_i\|_{HS} \leq \sup_{y\in\Y}\|\psi(y)\|_\HY \sup_{x\in\X}\|\varphi(x)\|_\HX  \leq \Q\kappa,
}
almost surely on the support of $\rho$ on $\X\times\Y$, and so $\mathbb{E} \|\zeta_i\|_{HS}^2 \leq Q^2\kappa^2$. Thus, by applying Lemma~2 of \cite{smale2007learning}, we have
\eqals{
\|\hat{Z}^*\hat{S} - Z^* S\|_{HS} \leq \frac{2 Q\kappa \log\frac{2}{\tau}}{n} + \sqrt{\frac{2Q^2\kappa^2\log\frac{2}{\tau}}{n}} \leq \frac{4 Q\kappa \log\frac{2}{\tau}}{n}
}
with probability $1 - \tau$, since $\log 2/\tau \geq 1$. To control $\| (\hat C + \la)^{-1} S^*\|$ we proceed by recalling that $C = S^* S$ and that for any $\la>0$ $\|(\hat{C}+\la)^{-1}\| \leq \la^{-1}$ and $\|\hat{C}(\hat{C}+\la)^{-1}\|\leq 1$. We have
\eqals{
\| (\hat C + \la)^{-1} S^*\| & = \| (\hat C + \la)^{-1} C (\hat C + \la)^{-1}\|^{1/2} \\
& \leq\| (\hat C + \la)^{-1} (C - \hat{C}) (\hat C + \la)^{-1}\|^{1/2} + \| (\hat C + \la)^{-1} \hat{C} (\hat C + \la)^{-1}\|^{1/2} \\
& \leq \| (\hat C + \la)^{-1}\| \|C - \hat{C}\|^{1/2} + \| (\hat C + \la)^{-1} \|^{1/2} \|\hat{C} (\hat C + \la)^{-1}\|^{1/2}  \\
& \leq \la^{-1/2}(1 + \la^{-1/2} \|C - \hat{C}\|^{1/2}).
}
To control $\|C - \hat{C}\|$, note that $\hat{C} = \frac{1}{n}\sum_{i=1}^n \zeta_i$ where $\zeta_i$ is the random variable defined as $\zeta_i = \varphi(x_i)\otimes \varphi(x_i)$ for $1 \leq i \leq n$. Note that ${\mathbb E} \zeta_i = C$, $\|\zeta_i\| \leq \kappa^2$ almost surely and so $\mathbb{E} \|\zeta_i\|^2 \leq \kappa^4$ for $1 \leq i \leq n$. Thus we can again apply Lemma~2 of \cite{smale2007learning}, obtaining
\eqals{
\|C - \hat{C}\| \leq \|C - \hat{C}\|_{HS} \leq \frac{2 \kappa^2 \log\frac{2}{\tau}}{n} + \sqrt{\frac{2 \kappa^4 \log\frac{2}{\tau}}{n}} \leq \frac{4\kappa^2 \log \frac{2}{\tau}}{\sqrt{n}},
}
with probability $1-\tau$. Thus, by performing an intersection bound, we have
\eqals{
A_1 \leq \frac{4Q\kappa\log\frac{2}{\tau}}{\sqrt{\la n}}\left(1 + \sqrt{\frac{4\kappa^2 \log \frac{2}{\tau}}{\la \sqrt{n}}}\right).
}
with probability $1 - 2\tau$.
The term $A_2$ can be controlled as follows
\eqals{
  A_2 &= \|Z^* S (\hat{C}+\la)^{-1} S^* - Z^* S (C+\la)^{-1} S^* \|_{HS}  \\
      & = \|Z^* S ( (\hat{C}+\la)^{-1} - (C+\la)^{-1} ) S^* \|_{HS} \\
    &= \|Z^* S (C +\la)^{-1} (C - \hat{C}) (\hat{C}+\la)^{-1} S^* \|_{HS} \\ 
  & \leq \|Z^* S (C +\la)^{-1}\|_{HS} \|C - \hat{C}\| \|(\hat{C} + \la)^{-1} S^*\| \\
  & = k^{-1} \mathcal{B}(\la) \|C - \hat{C}\| \|(\hat{C} + \la)^{-1} S^*\|
}
where we have used the fact that for two invertible operators $A$ and $B$ we have $A^{-1} - B^{-1} = A^{-1}(B-A)B^{-1}$. Now, by controlling $ \|C - \hat{C}\|,~\| (\hat C + \la)^{-1} S^*\|$ as for $A_1$ and performing an intersection bound, we have
\eqals{
A_2 & \leq \frac{4{\cal B}(\la)\kappa \log \frac{2}{\tau}}{\sqrt{\la n}}\left(1 + \sqrt{\frac{4\kappa^2 \log \frac{2}{\tau}}{\la \sqrt{n}}}\right)
}
with probability $1 - 2\tau$. Finally the term $A_3$ is equal to
\eqals{
  A_3 & = \|Z^* (S (C+\la)^{-1} S^* - I) \|_{HS} = \|Z^* (L (L+\la)^{-1} - I) \|_{HS} \\ 
  & = \|Z^* (L(L+\la)^{-1} - (L+\la)(L+\la)^{-1}) \|_{HS} = \la \|Z^* (L + \la)^{-1}\|_{HS} = \mathcal{A}(\la)
}
where $I$ denotes the identity operator. Thus, by performing an intersection bound of the events for $A_1$ and $A_2$, we have
\eqals{
\E(\hat{f}) - \E(f^*) \leq 8 \kappa Q \|V\| \frac{Q + {\cal B}(\la)}{\sqrt{\la n}}\left(1 + \sqrt{\frac{4\kappa^2}{\la \sqrt{n}}}\right)\log^2 \frac{2}{\tau} + 2 Q \|V\|{\cal A}(\la). 
}
with probability $1 - 4\tau$. Since $\delta = 4\tau$ we obtain the desired bound.
\end{proof}
Now we are ready to give the universal consistency result.
\Tuniversal*
\begin{proof}
By applying \autoref{lemma:prob-bound}, we have
\eqals{
\E(\hat{f}) - \E(f^*) \leq 8 \kappa Q \|V\| \frac{Q + {\cal B}(\la)}{\sqrt{\la n}}\left(1 + \sqrt{\frac{4\kappa^2}{\la \sqrt{n}}}\right)\log^2 \frac{8}{\delta} + 2 Q \|V\|{\cal A}(\la),
}
with probability $1-\delta$. Note that, since $C = S^*S$, $\|(C+\la)^{-1}\| \leq \la^{-1}$ and $\|C(C+\la)^{-1}\| \leq 1$, we have
\eqals{
\kappa^{-1} {\cal B}(\la) & = \|Z^*S(C+\la)^{-1}\|_{HS} \leq \|Z\|_{HS}\|S(C+\la)^{-1}\| \\
& \leq \|Z\|_{HS} \|(C+\la)^{-1} S^*S(C+\la)^{-1}\|^{1/2}\\
& \leq \|Z\|_{HS} \|(C+\la)^{-1}\|^{1/2}\|C(C+\la)^{-1}\|^{1/2} \\
& \leq \|Z\|_{HS} \la^{-1/2}. 
}
Therefore
\eqals{
\E(\hat{f}) - \E(f^*) \leq 8 \kappa Q \|V\| \frac{Q + \frac{\kappa}{\sqrt{\la}}\|Z\|_{HS} }{\sqrt{\la n}}\left(1 + \sqrt{\frac{4\kappa^2}{\la \sqrt{n}}}\right)\log^2 \frac{8}{\delta} + 2 Q \|V\|{\cal A}(\la),
}
Now by choosing $\la = \kappa^2 n^{-1/4}$, we have
\eqals{
\E(\hat{f}_n) - \E(f^*) \leq 24 Q \|V\| (Q + \|Z\|_{HS})n^{-1/4} \log^2\frac{8}{\delta}  + 2\Q\|V\|{\cal A}(\la), 
}
with probability $ 1- \delta$. Now we study ${\cal A}(\la)$, let $L = \sum_{i \in \N} \sigma_i \ u_i \otimes u_i$ be the eigendecomposition of the compact operator $L$, with $\sigma_i \geq \sigma_j > 0$  for $1 \leq i \leq j \in \N$ and $u_i \in L^2(\X,\rhox)$. Now, let $w_i^2 = \ip{u_i}{ZZ^* u_i}_{\LXR} = \int \ip{g^*(x)}{u_i}_\HY^2 d\rhox(x)$ for $i \in \N$. We need to prove that $(u_i)_{i \in \N}$ is a basis for $\LX$.
Let $W \subseteq \X$ be the support of $\rhox$, note that $W$ is compact and Polish since it is closed and subset of the compact Polish space $\X$. Let ${\cal L}$ be the RKHS defined by ${\cal L} = \overline{\Span\{k(x,\cdot)~|~x \in W\}}$, with the same inner product of $\HX$. By the fact that $W$ is a compact polish space and $k$ is continuous, then ${\cal L}$ is separable. By the universality of $k$ we have that ${\cal L}$ is dense in $C(W)$, and, by Corollary 5 of \cite{micchelli2006universal}, we have $C(W) = \overline{\Span\{u_i~|~i \in \N\}}$. Thus, since $C(W)$ is dense in $\LX$, we have $(u_i)_{i \in \N}$ is a basis of $\LX$.
 Thus $\sum_{i \in \N} w_i^2 = \int \|g^*(x)\|_\HY^2 d\rhox(x) = \|Z\|_{HS}^2 < \infty$. Therefore
\eqals{
{\cal A}(\la_n)^2 = \la_n^2 \| Z^* (L + \la_n)^{-1} \|_{HS}^2 = \sum_{i \in \N} \frac{\la_n^2 w_i^2}{(\sigma_i + \la)^2}.
}
Let $t_n = n^{-1/8}$, and $T_n = \{i \in \N \ | \ \sigma_i \geq t_n\} \subset \N$. For any $n\in\N$ we have
\eqals{
{\cal A}(\la_n) &= \sum_{i\in T_n} \frac{\la_n^2 w_i^2}{(\sigma_i + \la_n)^2} + \sum_{i \in \N\setminus T_n} \frac{\la_n^2 w_i^2}{(\sigma_i + \la_n)^2} \\
& \leq \frac{\la_n^2}{t_n^2}\sum_{i \in T_n} w_i^2 + \sum_{i \in \N\setminus T_n} w_i^2  \leq \kappa^4 \|Z\|_{HS}^2 ~ n^{-1/4} + \sum_{i\in\N\setminus T_n} w_i^2
}
since $\la_n/t_n = \kappa^2 n^{-1/4}/n^{-1/8} = \kappa^2 n^{-1/8}$. We recall that $L$ is a trace class operator, namely $\tr(L) = \sum_{i=1}^{+\infty} \sigma_i < +\infty$. Therefore $\sigma_i\to0$ for $i\to+\infty$, from which we conclude
\eqals{
0 \leq \lim_{n \to \infty} {\cal A}(\la_n) \leq \lim_{n \to \infty} \kappa^4 \|Z\|_{HS}^2 n^{-1/4} + \sum_{i \in \N\setminus T_n} w_i^2 = 0.
}
Now, for any $n\in\N$, let $\delta_n = n^{-2}$ and $E_n$ be the event associated to the equation
\eqals{
\E(\hat{f}_n) - \E(f^*) > 24 \Q \|V\| (\Q + \|Z\|_{HS})n^{-1/4} \log^2(8n^2)  + 2\Q \|V\|{\cal A}(\la).
}
By \autoref{lemma:prob-bound}, we know that the probability of $E_n$ is at most $\delta_n$. Since $\sum_{n=1}^{+\infty} \delta_n < +\infty$, we can apply the Borel-Cantelli lemma (Theorem 8.3.4. pag 263 of \cite{dudley2002real}) on the sequence $(E_n)_{n \in \N}$ and conclude that the statement
\eqals{
\lim_{n \to \infty} \E(\hat{f}_n) - \E(f^*) > 0,
}
holds with probability $0$. Thus, the converse statement
\eqals{
\lim_{n \to \infty} \E(\hat{f}_n) - \E(f^*) = 0.
}
holds with probability $1$.
\end{proof}

\subsection{Generalization Bounds}

Finally, we show that under the further hypothesis that $g^*$ belongs to the RKHS $\mG \simeq \HYX$, we are able to prove generalization bounds for the structured prediction algorithm.

\Tsimplebound*
\begin{proof}
By applying \ref{lemma:prob-bound}, we have
\eqals{
\E(\hat{f}) - \E(f^*) \leq 8 \kappa Q \|V\| \frac{Q + {\cal B}(\la)}{\sqrt{\la n}}\left(1 + \sqrt{\frac{4\kappa^2}{\la \sqrt{n}}}\right)\log^2 \frac{8}{\delta} + 2 Q \|V\|{\cal A}(\la),
}
with probability $1-\delta$. By assumption, $g^*\in\mG$ and therefore there exists a $G \in \HYX$ such that 
\eqals{
g^*(x) = G \varphi(x), \quad \forall x \in \X.
}
This implies that $Z^* = G S^*$ since, by definition of $Z$ and $S$, for any $h\in\LX$,
\eqals{
Z^*h = \int_\X g^*(x)h(x) d\rhox(x) = \int_\X G \varphi(x)h(x) d\rhox(x) = GS^*h.
}
Thus, since $L = S S^*$ and $\|(L+\la)^{-1} L\| \leq 1$ and $\|(L+\la)^{-1}\| \leq \la^{-1}$ for any $\la > 0$, we have
\eqals{
{\cal A}(\la) &= \la \|Z^* (L + \la)^{-1}\|_{HS} = \la \|G S^* (L + \la)^{-1}\|_{HS} \\
& \leq \la \|G\|_{HS} \|S^*(L+\la)^{-1}\| = \la \|G\|_{HS} \|(L+\la)^{-1} S S^*(L+\la)^{-1}\|^{1/2} \\
& \leq \la \|G\|_{HS} \|(L+\la)^{-1} L\|^{1/2} \|(L+\la)^{-1}\|^{1/2} \\
& \leq \la^{1/2} \|G\|_{HS}.
}
Moreover, since $C = S^* S$, we have
\eqals{
\kappa^{-1}{\cal B}(\la) & = \|Z^* S (C + \la)^{-1}\|_{HS} = \|G S^*S (C + \la)^{-1}\|_{HS} \\
& \leq \|G\|_{HS} \|C (C + \la)^{-1}\| \\
& \leq \|G\|_{HS}.
}
Now, let $\la = \kappa^2 n^{-1/4}$, we have
\eqals{
\E(\hat{f}) - \E(f^*) & \leq 24 \Q \|V\| (\Q + \kappa \|G\|_{HS}) n^{-1/4}\log^2 \frac{8}{\delta} + 2 \Q \|V\| \kappa n^{-1/4} \\
& \leq 24 \Q \|V\| (\Q + \kappa \|G\|_{HS} + \kappa) n^{-1/4}\log^2 \frac{8}{\delta} \\
& = c\tau^2 n^{-1/4}
}
with probability $1 - 8 e^{-\tau}$, where we have set $\delta = 8 e^{-\tau}$ and $c = 24 \Q \|V\| (\Q + \kappa \|G\|_{HS} + \kappa)$ to obtain the desired inequality.
\end{proof}

\section{Examples of Loss Functions}\label{sec:losses}

In this section we prove \autoref{teo:taxonomy} to show that a wide range of functions $\loss:\Y\times\Y\to\R$ useful for structured prediction learning satisfies the loss trick (\autoref{assumption:main}). In the following we state \autoref{teo:taxonomy}, then we use it to prove that all the losses considered in \autoref{example:general} satisfy \autoref{assumption:main}. Finally we give two lemmas, necessary to prove \autoref{teo:taxonomy} and then conclude with its proof.

\begin{restatable}{theorem}{Ttaxonomy}\label{teo:big_taxonomy}\label{teo:taxonomy}
Let $\Y$ be a set. A function $\loss: \Y \times \Y \to \R$ satisfy \autoref{assumption:main} when at least one of the following conditions hold:
\begin{enumerate}
\item $\Y$ is a finite set, with discrete topology.
\item $\Y = [0, 1]^d$ with $d \in \N$, and the mixed partial derivative $L(y,y') = \frac{\partial^{2d} \loss(y_1,\dots,y_d,y'_1,\dots,y'_d)}{\partial y_1,\dots,\partial y_d, \partial y'_1,\dots,\partial y'_d}$ exists almost everywhere, where $y = (y_i)_{i=1}^d, y' = (y'_i)_{i=1}^d \in \Y$, and satisfies
\eqal{\label{eq:mix-der-def}
\int_{\Y\times \Y} |L(y,y')|^{1+\epsilon} dydy' < \infty, \quad \textrm{with} \quad \epsilon > 0.
}
\item $\Y$ is compact and $\loss$ is a continuous kernel, or $\loss$ is a function in the RKHS induced by a kernel $K$. Here $K$ is a continuous kernel on $\Y \times \Y$, of the form 
$$K((y_1,y_2),(y_1',y_2')) = K_0(y_1,y_1') K_0(y_2,y_2'), \quad \forall y_i,y_i' \in \Y, i = 1,2,$$
with $K_0$ a bounded and continuous kernel on $\Y$.
\item $\Y$ is compact and
$$\Y \subseteq \Y_0, \quad \loss = \loss_0|_{\Y},$$ 
that is the restriction of $\loss_0:\Y_0 \times \Y_0 \to \R$ on $\Y$, and $\loss_0$ satisfies \autoref{assumption:main} on $\Y_0$,
\item $\Y$ is compact and 
$$\loss(y,y') = f(y)\loss_0(F(y),G(y'))g(y'),$$
with $F, G$ continuous maps from $\Y$ to a set $\cal Z$ with $\loss_0:{\cal Z}\times{\cal Z}\to\R$ satisfying \autoref{assumption:main} and $f, g: \Y \to \R$, bounded and continuous. 
\item $\Y$ compact and 
$$\loss = f(\loss_1,\dots,\loss_p),$$ 
where $f:[-M, M]^d \to \R$ is an analytic function (e.g. a polynomial), $p \in \N$ and $\loss_1,\dots,\loss_p$ satisfy \autoref{assumption:main} on $\Y$. Here $M \geq \sup_{1\leq i \leq p}\|V_i\|C_i$ where $V_i$ is the operator associated to the loss $\loss_i$ and $C_i$ is the value that bounds the norm of the feature map $\psi_i$ associated to $\loss_i$, with $i \in \{1,\dots,p\}$.
\end{enumerate}
\end{restatable}

Below we expand \autoref{example:general} by proving that the considered losses satisfy \autoref{assumption:main}. The proofs are typically a direct application of \autoref{teo:big_taxonomy} above. 
\begin{enumerate}
\item{\em Any loss with, $\Y$ finite.}~ This is a direct application of Thm.~\ref{teo:big_taxonomy}, point 1.
\item{\em Regression and classification loss functions.}~ Here $\Y$ is an interval on $\R$ and the listed loss functions satisfies Thm.~\ref{teo:big_taxonomy}, point 2. For example, let $\Y = [-\pi, \pi]$ the mixed partial derivative of the Hinge loss $\loss(y, y') = \max(0,1 - y y')$ is defined almost everywhere as $L(y,y') = -1$ when $y y' < 1$ and $L(y,y') = 0$ otherwise. Note that $L$ satisfies \eqref{eq:mix-der-def}, for any $\la > 0$.
\item{\em Robust loss functions.}~ Here, again $\Y$ is an interval on $\R$. The listed loss functions are: {\em Cauchy} $\gamma \log(1 + |y - y'|^2/\gamma)$, {\em German-McLure} $|y - y'|^2/(1 + |y - y'|^2)$ {\em ``Fair''} $\gamma |y - y'| - \gamma^2 \log(1 + |y - y'|/\gamma)$ or the {\em ``$\mathbf{L_2-L_1}$''} $\sqrt{1 + |y - y'|^2} - 1$. They are differentiable on $\R$, hence satisfy \autoref{teo:taxonomy}, point 2. The {\em Absolute value} $|y-y'|$ is Lipschitz and satisfies \autoref{teo:taxonomy}, point 2, as well. 
\item{\em KDE.}~ When $\Y$ is a compact set and $\loss$ is a kernel, the point 3 of \autoref{teo:big_taxonomy} is applicable.
\item{\em Diffusion Distances on Manifolds.}~
Let $M \in \N$ and $\Y\subset\R^M$ be a compact Reimannian manifold. The {\it heat kernel} (at time $t>0$), $k_t:\Y\times\Y\to\R$ induced by the Laplace-Beltrami operator of $\Y$ is a reproducing kernel \cite{schoen1994}. The {\em squared diffusion distance} is defined in terms of $k_t$ as follows $\loss(y,y') = 1 - k_t(y,y')$. Then, point 3 of \autoref{teo:taxonomy} is applicable.
\item{\em Distances on Histograms/Probabilities.}~
Let $M \in \N$. A discrete probability distribution (or a normalized histogram) over $M$ entries can be represented as a $y = (y_i)_{i=1}^M\in\Y\subset [0,1]^M$ the $M$-simplex, namely $\sum_{i=1}^M y_i= 1$ and $y_i\geq0$ $\forall i=1,\dots,M$. A typical measure of distance on $\Y$ is the squared {\it { Hellinger} (or Bhattacharya)} $\loss_H(y,y') = (1/\sqrt{2}) \|\sqrt{y} - \sqrt{y}'\|_2^2$, with $\sqrt{y} = (\sqrt{y_i})_{i=1}^M$. By \autoref{teo:taxonomy}, points 4, 6 we have that $\loss_H$ satisfies \autoref{assumption:main}. Indeed, consider the kernel $k$ on $\R$, $k(r,r') = (\sqrt{rr'} + 1)^2$ with feature map $\varphi(r) = (r, \sqrt{2r}, 1)^\top\in\R^3$, Then 
$$\loss_0(r,r') = (\sqrt{r} - \sqrt{r'})^2 = r - 2\sqrt{rr'} + r' = \varphi(r)^\top V \varphi(r') \quad \textrm{with}\quad V = \left(\begin{array}{ccc} 0 & 0 & 1 \\ 0 & -1 & 0 \\ 1 & 0 & 0 \end{array}\right).$$ $\loss_H$ is obtained by $M$ summations of $\loss_0$ on $[0,1]$, by \autoref{teo:taxonomy}, point 6 (indeed $\loss_H(y,y') = f(\loss_0(y_1,y'_1),\dots,\loss_0(y_M,y'_M))$ with $y = (y_i)_{i=1}^M, y' = (y'_i)_{i=1}^M \in \Y$ and the function $f: \R^M \to \R$ defined as $f(t_1,\dots, t_M) = \sum_{i=1}^M t_i$, which is analytic on $\R^M$), and then restriction on $\Y$ \autoref{teo:taxonomy}, point 4. A similar reasoning holds when the loss function is the $\chi^2$ distance on histograms. Indeed the function $(r - r')^2/(r + r')$ satisfies point 2 on $\Y = [0,1]$, then point 6 and 4 are applied.
\end{enumerate}

To prove \autoref{teo:big_taxonomy} we need the following two Lemmas.

\begin{lemma}[multiple Fourier series]\label{lemma:fseries}
Let $D = [-\pi,\pi]^d$,  $(\widehat  f_h)_{h \in \Z^d} \in \CC$ and $f: D  \to \CC$ with $d \in \N$ defined as
$$f(y) = \sum_{h \in \Z^d} \widehat f_h e^{i h^\top y}, \quad \forall y \in D, \quad \textrm{with}\quad \sum_{h \in \Z^d} |\widehat f_h| \leq B,$$
for a $B < \infty$ and $i = \sqrt{-1}$. Then the function $f$ is continuous and
$$\sup_{y \in D} |f(y)| \leq B.$$
\end{lemma}
\begin{proof}
For the continuity, see \cite{kahane1995fourier} (pag. 129 and Example 2). For the boundedness we have
$$ \sup_{y \in D} |f(y)| \leq \sup_{y \in D} \sum_{h \in \Z^d} |\widehat f_h| |e^{i h^\top y}| \leq \sum_{h \in \Z^d} |\widehat f_h| \leq B.$$
\end{proof}

\begin{lemma}\label{lemma:abs-cont-functions}
Let $\Y = [-\pi,\pi]^d$ with $d \in \N$, and $\loss: \Y \times \Y \to \R$ defined by
$$ \loss(y, z) = \sum_{h,k \in \Z^d} \widehat\loss_{h,k} e_h(y)e_k(z), \quad \forall y, z \in \Y,$$
with $e_h(y) = e^{i h^\top y}$ for any $y \in \Y$, $i = \sqrt{-1}$ and $\widehat\loss_{h,k} \in \CC$ for any $h,k \in \Z^d$. The loss $\loss$ satisfies \autoref{assumption:main} when
$$\sum_{h,k \in \Z^d} |\widehat\loss_{h,k}| < \infty.$$
\end{lemma}
\begin{proof}
Note that by applying \autoref{lemma:fseries}  with $D = \Y \times \Y$, the function $\loss$ is bounded continuous. Now we introduce the following sequences
$$ \alpha_{h} = \sum_{k \in \Z^d} |\widehat\loss_{h,k}|,  \quad f_h(z) = \frac{1}{\alpha_h} \sum_{h \in \Z^d} \widehat\loss_{hk} e_k(z) \quad \forall h \in \Z^d, z \in \Y.$$
Note that $(\alpha_h)_{h \in \Z^d} \in \ell_1$ and that $f_h$ bounded by $1$ and is continuous by \autoref{lemma:fseries} for any $h \in \Z^d$.
Therefore
$$ \loss(y, z) = \sum_{h,k \in \Z^d} \widehat\loss_{h,k} e_h(y)e_h(z) = \sum_{h \in \Z^d} \alpha_h e_h(y)f_h(z).$$
Now we define two feature maps $\psi_1,\psi_2: \Y \to \H_0$, with $\H_0 = \ell_2(\Z^d)$, as
$$\psi_1(y) = (\sqrt{\alpha_h}e_h(y))_{h \in \Z^d},\quad \psi_2(y) = (\sqrt{\alpha_h}f_h(y))_{h \in \Z^d}, \quad \forall y \in \Y.$$
Now we prove that the two feature maps are continuous. Define $k_1(y,z) = \ip{\psi_1(y)}{\psi_1(z)}_{\H_0}$ and $k_2(y,z) = \ip{\psi_2(y)}{\psi_2(z)}_{\H_0}$  for all $y,z \in \Y$. We have 
\eqals{
k_1(y,z) &= \sum_{h \in \Z^d} \alpha_h \overline{e_h(y)}e_h(z),\\
k_2(y,z) &= \sum_{h \in \Z^d} \alpha_h \overline{f_h(y)} f_h(z) = \sum_{k,l \in \Z^d} \beta_{k,l} \overline{e_k(y)} e_l(z)
}
with $\beta_{k,l} = \sum_{h \in \Z^d} \frac{\overline{\widehat\loss_{h,k}}\widehat\loss_{h,l}}{\alpha_h}$, for $k, l \in \Z^d$, therefore $k_1$, $k_2$ are bounded and continuous by \autoref{lemma:fseries} with $D = \Y \times \Y$, since $\sum_{h \in \Z^d} \alpha_h < \infty$ and $\sum_{k,l \in \Z^d} |\beta_{k,l}| < \infty$.
Note that $\psi_1$ and $\psi_2$ are bounded, since $k_1$ and $k_2$ are. Moreover for any $y, z \in \Y$, we have
\eqals{
\|\psi_1(y) - \psi_1(z)\|^2 &= \ip{\psi_1(z)}{\psi_1(z)}_{\H_0} + \ip{\psi_1(y)}{\psi_1(y)}_{\H_0} - 2 \ip{\psi_1(z)}{\psi_1(y)}_{\H_0} \\
& = k_1(z,z) + k_1(y,y) - 2 k_1(z,y) \leq |k_1(z,z) - k_1(z,y)| + |k_1(z,y) - k_1(y,y)|,
}
and the same holds for $\psi_2$ with respect to $k_2$.
Thus the continuity of $\psi_1$ is entailed by the continuity of $k_1$ and the same for $\psi_2$ with respect to $k_2$.
Now we define $\HY = \H_0 \oplus \H_0$ and $\psi: \Y \to \HY$ and $V: \HY \to \HY$ as
$$ \psi(y) = (\psi_1(y), \psi_2(y)),~~\forall y \in \Y \quad \textrm{and} \quad V = \begin{pmatrix} 0 & I\\ 0 & 0 \end{pmatrix},$$
where $I: \H_0 \to \H_0$ is the identity operator.
Note that $\psi$ is bounded continuous, $V$ is bounded and
$$\ip{\psi(y)}{V\psi(z)}_\HY = \ip{\psi_1(y)}{\psi_2(z)}_{\H_0} = \sum_{h \in \Z^d} \alpha_h e_h(y) f_h(z) = \loss(y,z).$$
\end{proof}

We can now prove \autoref{teo:big_taxonomy}. 

\begin{proof}(\autoref{teo:big_taxonomy})

\begin{enumerate}
\item[1] Let $N = \{1,\dots,|\Y|\}$ and $q: \Y \to N$ be a one-to-one function. Let $\HY = \R^{|\Y|}$, $\psi(y):\Y \to \HY$ defined by $\psi(y) = e_{q(y)}$ for any $y \in \Y$ with $(e_i)_{i=1}^{|\Y|}$ the canonical basis for $\HY$, finally $V \in \R^{|\Y|\times|\Y|}$ with $V_{i,j} = \loss(q^{-1}(i),q^{-1}(j))$ for any $i,j \in N$. Then $\loss$ satisfies \autoref{assumption:main}, with $\psi$ and $V$.
\item[2] By \autoref{lemma:abs-cont-functions}, we know that any loss $\loss$ whose Fourier expansion is absolutely summable, satisfies \autoref{assumption:main}. The required conditions in points 2 are sufficient ( see Theorem~6' pag.~291 of \cite{moricz2007absolute}).
\item[3] Let $\Y$ be a compact space. For the first case let $\loss$ be a bounded and continuous reproducing kernel on $\Y$ and let $\HY$ the associated RKHS, then there exist a bounded and continuous map $\psi:\Y \to \HY$ such that $\loss(y,y') = \ip{\psi(y)}{\psi(y')}$ for any $y, y' \in \Y$, which satisfies \autoref{assumption:main} with $V = I$ the identity on $\HY$. For the second case, let $K$ defined in terms of $K_0$ as in equation. Let $\H_0$ be the RKHS induced by $K_0$ and $\psi$ the associated feature map, then, by definition, the $RKHS$ induced by $K$ will be $\H = \H_0 \otimes \H_0$. Since $\loss$ belongs to $\H$, then there exists a $v \in \H$ such that $\loss(y,y') = \ip{v}{\psi(y)\otimes\psi(y')}_{\H_0\otimes\H_0}$. Now note that $\H = \H_0 \otimes \H_0$ is isomorphic to $B_2(\H_0, \H_0)$, that is the linear space of Hilbert-Schmidt operators from $\H_0$ to $\H_0$, thus, there exist an operator $V \in B_2(\H_0, \H_0)$ such that 
$$\loss(y,y') = \ip{v}{\psi(y)\otimes\psi(y')}_{\H_0\otimes\H_0} = \ip{\psi(y)}{V\psi(y')}_{\H_0},\quad \forall y, y' \in \Y.$$
Finally note that $\psi$ is continuous and bounded, since it is $K_0$, and $V$ is bounded since it is Hilbert-Schmidt. Thus $\loss$ satisfies \autoref{assumption:main}.
\item[4] Since $\loss_0$ satisfies \autoref{assumption:main} we have that there exists a kernel on $\Y_0$ such that \autoref{eq:linearized_loss} holds. Note that the restriction of a kernel on a subset of its domain is again a kernel. Thus, let $\psi = \psi_0|_{\Y}$, we have that $\loss|_{\Y}$ satisfies \autoref{eq:linearized_loss} with $\psi$ and the same bounded operator $V$ as $\loss$.
\item[5] Let $\loss_0:{\cal Z} \times {\cal Z} \to \R$ satisfy \autoref{assumption:main}, then $\loss_0(z,z') = \ip{\psi_0(z)}{V_0\psi_0(z')}$ for all $z, z' \in {\cal Z}$ with $\psi_0: \Y \to \H_0$ bounded and continuous and $\H_0$ a separable Hilbert space. Now we define two feature maps $\psi_1(y) = f(y)\psi_0(F(y))$ and $\psi_2(y) = g(y)\psi_0(G(y))$. Note that both $\psi_1, \psi_2: \Y \to \H_0$ are bounded and continuous. We define $\HY = \H_0 \oplus \H_0$, $\psi: \Y \to \HY$ as $\psi(y) = (\psi_1(y), \psi_2(y'))$ for any $y \in \Y$, and $V = \begin{pmatrix}0 & V_0\\ 0 & 0 \end{pmatrix}$. Note that now
\eqals{
\ip{\psi(y)}{V \psi(y')}_{\HY} &= (\psi_1(y)~ \psi_2(y))  \begin{pmatrix}0 & V_0\\ 0 & 0 \end{pmatrix}  \begin{pmatrix}\psi_1(y)\\ \psi_2(y)\end{pmatrix} = \ip{\psi_1(y)}{V_0 \psi_2(y')}_{\HY} \\
& = f(y)\ip{\psi(F(y))}{V_0 \psi(G(y'))}_{\HY}g(y') =   f(y)\loss_0(F(y),G(y'))g(y').
}
\item[6] Let $\Y$ be compact and $\loss_i$ satisfies \autoref{assumption:main} with $\H_i$ the associated RKHS, with continuous feature maps $\psi_i:\Y \to \H_i$  bounded by $C_i$ and with a bounded operator $V_i$, for $i \in \{1,\dots,p\}$.
Since an analytic function is the limit of a series of polynomials, first of all we prove that a finite polynomial in the losses satisfies \autoref{assumption:main}, then we take the limit.
First of all, note that $\alpha \loss_1 + \beta \loss_2$, satisfies \autoref{assumption:main}, for any $\alpha_1, \alpha_2 \in \R$. Indeed we define $\HY = \H_1 \oplus \H_2$, and $\psi(y) = (\sqrt{|\alpha|\|V_1\|}\psi_1(y),~~\sqrt{|\beta|\|V_2\|}\psi_2(y))$ for any $y \in \Y$, so that
$$\alpha_1 \loss_1(y,y') + \alpha_2 \loss_2(y,y') = \ip{\psi(y)}{V \psi(y')}_{\HY}, \quad \textrm{with}\quad V = \begin{pmatrix}  \frac{\textrm{sign}(\alpha_1)}{\|V_1\|}V_1 & 0 \\ 0 & \frac{\textrm{sign}(\alpha_2)}{\|V_2\|}V_2 \end{pmatrix},$$
for any $y, y' \in \Y$, where $\psi$ is continuous, $\sup_{y \in \Y} \|\psi\|_{\HY} \leq |\alpha_1| \|V_1\| C_1 + |\alpha_2| \|V_2\|C_2$ and $\|V\| \leq 1$. In a similar way we have that $\loss_1 \loss_2$ satisfies \autoref{assumption:main}, indeed, we define $\HY = \H_1 \otimes \H_2$ and $\psi$ to be $\psi(y) = \psi_1(y) \otimes \psi_2(y)$ for any $y \in \Y$,
thus
$$\loss_1(y,y')\loss_2(y,y') = \ip{\psi(y)}{V\psi(y)}_{\HY}, \quad \textrm{with} \quad V = V_1 \otimes V_2,$$
for any $y, y' \in \Y$, where $\psi$ is continuous, $\sup_{y \in \Y} \|\psi\|_{\HY} \leq C_1C_2$ and $\|V\| \leq \|V_1\|\|V_2\|$.
Given a polynomial $P(\loss)$, with $\loss = 
(\loss_1,\dots,\loss_p)$ we write it as
$$P(\loss) = \sum_{t \in \N^p} \alpha_t \loss^t, \quad \textrm{with} \quad \loss^t = \prod_{i=1}^p \loss_i^{t_i},~~\forall t \in \N^p.$$
where the $\alpha$'s are the coefficents of the polynomial and such that only a finite number of them are non-zero. By applying the construction of the product on each monomial and of the sum on the resulting monomials, we have that
$P(\loss)$ is a loss satisfying \autoref{assumption:main} for a continuous $\psi$ and a $V$ such that
$$\sup_{y \in \Y} \|\psi\|_{\HY} \leq \bar{P}(\|V_1\|C_1,\dots,\|V_1\|C_1)$$ 
and $\|V\| \leq 1$, where 
$\bar{P}(\loss) = \sum_{t \in \N^p} |\alpha_t| \loss^t$ and $\HY = \oplus_{t \in \N^p} ~ \otimes_{i=1}^p \H_i^{t_i}$. Note that $\HY$ is again separable.
Let now consider
$$f(\loss) = \sum_{t \in \N^p} \alpha_t \loss^t,\quad \bar{f}(\loss) = \sum_{t \in \N^p} |\alpha_t| \loss^t.
$$
Assume that $\bar{f}(\|V_1\|C_1,\dots, \|V_p\|C_p) < \infty$. Then by repeating the construction for the polynomials, we produce a bounded $\psi$ and a bounded $V$ such that 
$$f(\loss_1(y,y'),\dots,\loss_p(y,y')) = \ip{\psi(y)}{V \psi(y')}_{\HY}, \quad \forall ~ y,y' \in \Y,$$
in particular $\HY$ is the same for the polynomial case and $\psi = \oplus_{t \in \N^p} \sqrt{|\alpha_t|v_t} \otimes_{i=1}^p \psi_{i}^{\otimes t_i}$, with $v_t = \prod_{i=1}^p \|V_i\|^{t_i}$, for any $t \in \N^p$.
Now we prove that $\psi$ is continuous on $\Y$. Let $\psi_q$ be the feature map defined for the polynomial $\bar{P}_q(\loss) = \sum_{t \in \N^p,~\|t\| \leq q} |\alpha_t| \loss^t$. We have that
\eqals{
\sup_{y \in \Y} \|\psi(y) - \psi_q(y)\|^2_{\HY}  & = \sup_{y \in \Y} \left\|\oplus_{t \in \N^p, \|t\| > q} \sqrt{|\alpha_t|v^t} \otimes_{i=1}^p \psi_{i}^{\otimes t_i}\right\|^2_{\HY}  \\
& \leq  \sum_{t \in \N^d, \|t\| > q} |\alpha_t| v_t \sup_{y \in \Y} \|\otimes_{i=1}^p \psi_{i}^{\otimes t_i}\|^2_{\HY} \leq \sum_{t \in \N^d, \|t\| > q} |\alpha_t| v_t c_t,
}
with $c_t = \prod_{i=1}^p C_i^{t_i}$ for any $t \in \N^p$. No note that 
$$ \sum_{t \in \N^d} |\alpha_t| v_t c_t = \bar{f}(\|V_1\|C_1,\dots,\|V_p\|C_p) < \infty,
$$
thus $\lim_{q \to \infty} \sum_{t \in \N^d, \|t\| > q} |\alpha_t| v_t c_t = 0$.
Therefore
\eqals{
\lim_{q \to \infty} \sup_{y \in \Y} \|\psi(y) - \psi_q(y)\|^2_{\HY} &\leq \lim_{q \to \infty} \sum_{t \in \N^p,~\|t\|> q} |\alpha_t| v_t c_t = 0.
}
Now, since $\psi_q$ is a sequence of continuous bounded functions, and the sequence converges uniformly to $\psi$, then $\psi$ is continuous bounded. So $f(\loss)$ is a loss function satisfying \autoref{assumption:main}, with a continuous $\psi$ and an operator $V$ such that
$$
\sup_{y \in \Y} \|\psi\|_{\HY} \leq \bar{f}(\|V_1\|C_1,\dots,\|V_1\|C_1),
$$ and
$\|V\| \leq 1$.
\end{enumerate}
\end{proof}

\putbib[biblio]
\end{bibunit}

\end{document}